\documentclass[a4paper,11pt]{article}
\usepackage{mattsstyle}

\newcommand{\dWinfty}{\mathrm{d}_{\mathrm{W}^\infty}}
\newcommand{\cEn}{\cE_n}
\newcommand{\cEntau}{\cE_{n,\tau}}
\newcommand{\cEnyn}{\cEntau^{(\mathbf{y}_n)}}
\newcommand{\cEngn}{\cEntau^{(\mathbf{g}_n)}}

\newcommand{\cEinftytau}{\cE_{\infty,\tau}}
\newcommand{\cEinftyg}{\cEinftytau^{(g)}}

\newcommand{\Lany}{\textrm{\textnormal{L}}}
\newcommand{\Lp}{\textrm{\textnormal{L}}^p}
\newcommand{\Linfty}{\textrm{\textnormal{L}}^\infty}
\newcommand{\Ltwo}{\textrm{\textnormal{L}}^2}

\newcommand{\SobSpace}{\textrm{\textnormal{H}}}
\newcommand{\Hs}{\SobSpace^s}
\newcommand{\Cinfty}{\textrm{\textnormal{C}}^\infty}

\definecolor{mygreen}{rgb}{0.1,0.75,0.2}

\def\Lp#1{\mathrm{L}^{#1}}

\def\Hk#1{\mathrm{H}^{#1}}
\def\Ck#1{\mathrm{C}^{#1}}

\def\trace{\mathrm{trace}}

\title{Rates of Convergence for Regression with the Graph Poly-Laplacian}
\author[1]{Nicol\'{a}s Garc\'{i}a Trillos}
\author[2]{Ryan Murray}
\author[3]{Matthew Thorpe}
\affil[1]{Department of Statistics,\protect\\ University of Wisconsin--Madison,\protect\\ Madison, WI 53706, USA. \vspace{\baselineskip}}
\affil[2]{Department of Mathematics,\protect\\ North Carolina State University,\protect\\ Raleigh, NC 27695, USA. \vspace{\baselineskip}}
\affil[3]{School of Mathematics,\protect\\ University of Manchester,\protect\\ Manchester, M13 9PL, UK.}
\date{September 2022}

\begin{document}

\maketitle

\begin{abstract}
In the (special) smoothing spline problem one considers a variational problem with a quadratic data fidelity penalty and Laplacian regularisation.
Higher order regularity can be obtained via replacing the Laplacian regulariser with a poly-Laplacian regulariser.
The methodology is readily adapted to graphs and here we consider graph poly-Laplacian regularisation in a fully supervised, non-parametric, noise corrupted, regression problem.
In particular, given a dataset $\{x_i\}_{i=1}^n$ and a set of noisy labels $\{y_i\}_{i=1}^n\subset\bbR$ we let $u_n:\{x_i\}_{i=1}^n\to\bbR$ be the minimiser of an energy which consists of a data fidelity term and an appropriately scaled graph poly-Laplacian term.
When $y_i = g(x_i)+\xi_i$, for iid noise $\xi_i$, and using the geometric random graph, we identify (with high probability) the rate of convergence of $u_n$ to $g$ in the large data limit $n\to\infty$.
Furthermore, our rate, up to logarithms, coincides with the known rate of convergence in the usual smoothing spline model.
\end{abstract}

\noindent

\noindent
% 	49J55  	Problems involving randomness [See also 93E20]
% 	49J45  	Methods involving semicontinuity and convergence; relaxation
% 	68R10  	Discrete mathematics in relation to computer science: Graph theory
% 	62G20  	Nonparametric inference: Asymptotic properties
%   60D05   Geometric probability and stochastic geometry
%  	35J20   Variational methods for second-order elliptic equations
%   65C60   Computational problems in statistics
% 	65N12   Stability and convergence of numerical methods

\section{Introduction} \label{sec:Intro}

Given the applications to signal processing and computer science the smoothing spline problem has been attracting the interest of Statisticians since the 1960's~\cite{schoenberg64,schoenberg64a}.
The problem can stated as~\cite{wahba1990spline}: given feature vectors $\{x_i\}_{i=1}^n\subset\Omega\subset\bbR^d$ and labels $\{y_i\}_{i=1}^n\subset\bbR$ minimise
\begin{equation} \label{eq:Intro:Spline}
\cE_n^{(\mathrm{spline})} (u) = \frac{1}{n} \sum_{i=1}^n |y_i - u(x_i)|^2 + \tau \|\nabla^s u\|_{\Lp{2}(\Omega)}^2
\end{equation}
over all $u\in\Hk{s}(\Omega)$ where $\Hk{s}(\Omega)$ is the Sobolev space with square integrable $s$th (weak) derivative on $\Omega$.
Here, one tries to find an unknown function $u:\Omega\to\bbR$ that is trying to match the observed labels $\{y_i\}_{i=1}^n$ at $\{x_i\}_{i=1}^n$ whilst being smooth (in the sense of $\Hk{s}$ regularisation).
It is important to note that the regularity penalty is applied uniformly throughout the domain $\Omega$.
\vspace{\baselineskip}

Recently, the spline methodology has found use in data science and machine learning as a candidate for semi-supervised or fully-supervised learning.
Zhu, Ghahramani and Lafferty~\cite{zhu03} introduced the following variational problem as a method for finding missing labels.
They assumed that for every pair of feature vectors $x_i,x_j$ with $i,j\in\{1,\dots,n\}$ one has a measure of similarity $W_{i,j}$.
Further assuming that there is no error in the observed labels $\{y_i\}_{i\in Z_n}$, where $Z_n\subsetneq\{1,\dots,n\}$, they proposed the variational problem: minimise
\begin{equation} \label{eq:Intro:ZGL}
\cE_n^{(\mathrm{ZGL})}(u) = \sum_{i,j=1}^n W_{i,j} |u(x_i) - u(x_j)|^2
\end{equation}
subject to $u(x_i) = y_i$ for all $i\in Z_n$ over $u:\{x_i\}_{i=1}^n\to \bbR$.
The power of this method is that the regularisation will be applied more strongly when the density of data is higher.
Indeed, the continuum, $n\to\infty$, limit of~\eqref{eq:Intro:ZGL} is, up to a multiplicative constant,
\begin{equation} \label{eq:Intro:ZGLLimit}
\cE_\infty(u) = \int_\Omega |\nabla u(x)|^2 \rho^2(x) \, \dd x
\end{equation}
where $\rho$ is the density of the data.
We see that where $\rho$ is large the minimiser of~\eqref{eq:Intro:ZGLLimit} should be smoother, and conversely when $\rho$ is smaller the minimiser can fluctuate more.
This is typically desirable behaviour in classification tasks since the minimiser can be expected to be approximately constant within clusters and quickly transitioning outside of clusters where the density of data is assumed to be low.
\vspace{\baselineskip}

Several types of convergence result connecting~\eqref{eq:Intro:ZGL} to~\eqref{eq:Intro:ZGLLimit} exist in the literature.
For instance pointwise convergence of the objective functionals was established in~\cite{nadler09,elalaoui16}.
Rates of convergence between constrained minimisers of~\eqref{eq:Intro:ZGL} to constrained minimisers of~\eqref{eq:Intro:ZGLLimit} appeared in~\cite{calder20}.
Further results concern the convergence of the Laplacian operator, without rates in~\cite{vonLuxburg-belkin,belkin2007convergence,GTSSpectralClustering} and with rates in~\cite{HeinvonLuxburgAudibert,Singer,SpecRatesTrillos,calder19}, the game theoretic Laplacian in~\cite{calder18}, the $p$-Laplacian in~\cite{SlepcevThorpe}, and the $\infty$-Laplacian in~\cite{calder2017consistency}.
However, none of these results consider \emph{asymptotic consistency} in the sense that the minimiser of~\eqref{eq:Intro:ZGL} converges to a ``true function''.
(It would be more accurate to describe the above results as \emph{convergence} properties of the method.)
To our knowledge we are the first to consider consistency in the graph-based setting.
\vspace{\baselineskip}

When there is uncertainty in the observed labels then minimising~\eqref{eq:Intro:ZGL} with constraints is not the natural model.
Instead, as in~\eqref{eq:Intro:Spline}, we can write a \emph{soft} version of the Zhu, Ghahramani and Lafferty model by
\begin{equation} \label{eq:Intro:ZGLSoft}
\cEnyn(u) = \frac{1}{n} \sum_{i=1}^n |u(x_i) - y_i|^2 + \frac{\tau}{n^2\eps^2} \sum_{i,j=1}^n W_{i,j} |u(x_i) - u(x_j)|^2
\end{equation}
where we include the correct scaling on the second term.
The parameter $\tau$ controls the weighting between regularity and matching the data: formally $\tau=0$ corresponds to the hard constrained problem.
In some settings the large data limits for minimisers of~\eqref{eq:Intro:ZGLSoft} can be inferred from the \emph{hard} constrained problem, see~\cite{calder20}.
\vspace{\baselineskip}

To complete the generalisation of the Zhu, Ghahramani and Lafferty model to an analogue of the spline problem~\eqref{eq:Intro:Spline} we discuss higher order regularisation.
The Dirichlet energy $\cE_n^{(\textrm{ZGL})}$ can be written in inner product form:
\[ \cE_n^{(\textrm{ZGL})}(u) = \langle \Delta_n u, u\rangle_{\Lp{2}(\mu_n)} \]
where $\Delta_n$ is the graph Laplacian and $\Lp{2}(\mu_n)$ is the space of $\Lp{2}$ functions with respect to the empirical measure $\mu_n$ (defined in the following subsection).
A natural method to introduce higher order regularity is to consider higher powers of the graph-Laplacian, in particular
\begin{equation} \label{eq:Intro:ZGLSoftHigher}
\cEnyn(u) = \frac{1}{n} \sum_{i=1}^n (u(x_i) - y_i)^2 + \tau \langle \Delta_n^s u, u\rangle_{\Lp{2}(\mu_n)}
\end{equation}
in which case~\eqref{eq:Intro:ZGLSoft} is the special case of~\eqref{eq:Intro:ZGLSoftHigher} with $s=1$.
In the context of graphs this model was introduced in~\cite{dunlop18}, although the fractional Laplacian, including non-local and discrete versions, has been of interest in applied mathematics for much longer, see for example~\cite{ciaurri15} and references therein.
As observed in~\cite{zhu03} and (in terms of uncertainty quantification)~\cite{bertozzi18},~\eqref{eq:Intro:ZGLSoftHigher} has an interpretation as a maximum a-posteriori (MAP) estimate for the Gaussian process regression method (also known as kriging). 
As an aside we mention works that have explored the connection between discrete and continuum problems in the Bayesian setting such as \cite{SanzAlonsoTrillos,KaplanSamak} and~\cite{SanzAlonsoYangMatern}, where the latter introduces Mat\'ern priors on graphs and studies their continuum limits.
In this paper we use the setting of~\cite{dunlop18} to recover noisy observations of a labeling function $g$ as the MAP estimator given the data $\{(x_i,y_i)\}_{i=1}^n$ and using the graph poly-Laplacian to define a prior.
Our focus will be on recovering the labels $g$ \textit{as well as} some of its higher order information at the data points $\{ x_i \}_{i=1}^n$ in the form of powers of the Laplacian of $g$.
\vspace{\baselineskip}

Theoretical analysis of splines dates back to the 1960's and we refer to~\cite{wahba1990spline} for an overview of more historical works and only mention a few select (more recent) references here related to large data limits.
Convergence in norm of special splines under various settings have been studied in~\cite{bissantz04,bissantz07,claeskens09,hall05,kauermann09,lai13,lukas06,wang11,arcangeli93} and pointwise convergence results in~\cite{li08,shen11,xiao12,yoshida12,yoshida14}.
The general splines problem writes~\eqref{eq:Intro:Spline} in a more abstract framework.
In particular, one seeks to find $u\in \cH$, where $\cH$ is a reproducing kernel Hilbert space that can be decomposed $\cH = \cH_0\oplus\cH_1$, as the minimiser of
\[ \cE_n^{(\mathrm{gen}\,\mathrm{spline})} (u) = \frac{1}{n} \sum_{i=1}^n |y_i - L_iu|^2 + \tau \|P_1 u\|_{\cH}^2 \]
where $L_i\in\cH^*$ and $P_1:\cH\to\cH_1$ is the orthogonal projection.
For an appropriate choice of $\cH$ and $L_i$ one recovers the special smoothing spline problem (as a special case of the general smoothing spline problem).
The general smoothing spline problem has itself attracted attention with large data convergence in norm results appearing in~\cite{wahba85,kimeldorf70,cox88,nychka89,carroll91,mair96} and weak convergence results in~\cite{thorpe17}.
\vspace{\baselineskip}

In this paper, using the model~\eqref{eq:Intro:ZGLSoftHigher}, we consider the problem of non-parametric regression of a noisy signal $g$ observed at finitely many points that are randomly selected from an \textit{unknown} probability distribution $\mu$ on the torus $\cT$.
More precisely, we assume we are given a set of feature vectors $\{x_i\}_{i=1}^n\subset\Omega:=\cT$ and a set of associated noisy real valued labels $\{y_i\}_{i=1}^n$ satisfying
\begin{equation} \label{eq:Intro:Labels}
y_i = g(x_i)+ \xi_i,
\end{equation}
from which we wish to recover the true signal $g$, also known as the label function.
The random variables $\xi_i$ are assumed to be mean zero, sub-Gaussian and independent.
Our main results establish variance and bias estimates for the error of approximation of the signal $g$ and of some of its higher order information by the solution $u$ of a variational problem characterized by a graph PDE of the form
\begin{equation} \label{eq:Intro:GraphPDE}
\tau \Delta_n^s u  + u = y.
\end{equation}
Up to logarithms we establish an $O(n^{-\frac{s}{d+4s}})$ rate of convergence of minimisers of $\cEnyn$ in~\eqref{eq:Intro:ZGLSoftHigher} to $g$.
This is comparable to the $O(n^{-\frac{s}{d+2s}})$ rate of convergence in splines, see~\cite{stone82}.
Projecting the samples onto the first $K$ eigenvectors of the graph Laplacian gives a better rate of $O(n^{-\frac{2s}{2s+d}})$ (the minimax rate)~\cite{green2021minimax}, however $K$ depends on the $\Hk{s}$ norm of $g$ which will usually be unknown.
See Remarks~\ref{rem:Intro:Res:ErrorL2:Splines} and~\ref{rem:Intro:Res:ErrorL2:Tibshirani} for further details.
\vspace{\baselineskip}

Methods such as kernel ridge regression are closely related to smoothing spline models.
These methods use a data fidelity (on $\{x_i\}_{i=1}^n$) plus regularisation on $\Omega$ (in particular incorporating $\Omega$'s geometry \textit{explicitly}) to attempt to recover $g$.
In that setting, the question of how to set regularisation parameters was studied in \cite{cucker2002best,caponnetto2007optimal}.
Some very recent works have focused on studying the ``ridgeless case'', where one considers the limit as one sets the regularization parameter to zero, with both positive \cite{liang2018just} and negative \cite{rakhlin2018consistency} results depending on the richness of the data.
A related approach is to interpolate between data points using $\Ck{m}$ penalisation~\cite{fefferman09,fefferman09a,fefferman09b} or Sobolev penalisation~\cite{fefferman16,fefferman16a,fefferman16b}. 

There are connections between the Gaussian process regression (kriging) method approach that we take here and the generalised lasso model (which includes the lasso, the fused lasso, trend filtering, and the more closely related to our work: graph fused lasso, graph trend filtering, and Kronecker trend filtering), see for example~\cite{tibshirani11}.
Both Gaussian process regression and generalised lasso attempt to recover an unknown function $g$ from noisy observations of the form~\eqref{eq:Intro:Labels} in the fully supervised setting (i.e. for each feature vector $x_i$ we have an observation $y_i$).
However, in the lasso models the function $g$ is assumed to be linear, i.e. $g(x) = \beta\cdot x$ where $\beta$ is an unknown vector which parametrises $g$.
The fused lasso, on the other hand, uses a total variation regularisation in place of the graph poly-Laplacian considered here.
In grid graphs this has been considered in~\cite{hutter2016optimal,sadhanala2016total,sadhanala2017higher} where the estimator is shown to be minimax (the estimator performs best amongst all other estimators in the worst case).
Further results have considered chain graphs~\cite{padilla18}, and $k$-NN and $\eps$-connected graphs~\cite{padilla2019adaptive} (the $\eps$-connected graph setting is also the setting of this paper).
In particular, the $\Ltwo$ convergence rate of the fused lasso on an $\eps$-connected graph is (ignoring logarithms) $O(n^{-\frac{1}{d}})$ which at least in some settings is the minimax rate~\cite{padilla2019adaptive}.
Up to logarithms, and assuming a sufficiently smooth signal $g$, our basic $\Ltwo$ convergence rates for the approximation of $g$ coincide with these rates.
This is also approximately the $\Linfty$ convergence rate given in \cite{GTM-Regression} for the case $s=1$ in~\eqref{eq:Intro:GraphPDE}, which is the minimax $\Linfty$ rate given in \cite{Kpotufe}.
At this point we would like to remark that our variance estimates are meaningful and converge to zero with growing $n$ even if the regularization parameter $\tau$ is not scaled down to zero.
Our results characterize precisely the continuum limit of the solutions to the graph PDE~\eqref{eq:Intro:GraphPDE}, and are of relevance in case one were interested, not only on denoising, but also in enforcing additional regularization.
We also remark that in our results we provide additional information about the convergence towards $g$, by giving convergence rates for higher-order derivatives.

Other approaches for high order regularisation that do not consider Gaussian priors use instead a non-linear $p$-Laplacian operator for large enough $p$.
In the graph setting, results like those in \cite{hafiene-variational} establish that solutions of a $p$-Dirichlet regularised problem converge with rate $n^{-\frac14}$ to the solution of an analogue continuum non-local variational problem; although the setting differs from ours as we scale the connectivity of our graph to obtain a local limit whilst in~\cite{hafiene-variational} the connectivity of the graph remains fixed and the limit is to a nonlocal variational problem. 
Naturally, the advantage of the framework in \cite{hafiene-variational} is that the dimension of the space essentially plays no role in the analysis (depending on the precise edge model one uses) and therefore it is enough to consider the problem in 1D (as the authors do).
On the other hand, by not scaling down the connectivity threshold it is not possible to recover the local geometry.
The same authors, in the same setting, show a rate of convergence for the associated gradient flow~\cite{Hafiene-evolution-2018}.
As was mentioned earlier when discussing generalised Lasso models (in particular in graph trend filtering), total variation is another tool used to regularise regression and classification problems.
This has motivated theoretical works like~\cite{garciatrillos16} which study the convergence of graph total variation to a continuum weighted total variation (the same paper proposed a topology to study the convergence that didn't require regularity --- in particular pointwise evaluation --- of the continuum function).
Total variation functionals are also widely used for clustering and segmentation such as in graph cut methods, for example ratio or Cheeger cuts~\cite{BressonSzlam,garciatrillos20}, graph modularity clustering~\cite{BertozziModularity,davis18}, and Ginzburg--Landau segmentation~\cite{cristoferi18AAA,thorpe19,vangennip12a}.

Since we have observations for all feature vectors our problem is in the fully-supervised setting.
The semi-supervised setting with Laplacian regularisation (closely related to~\eqref{eq:Intro:GraphPDE} with $s=1$ but with hard constraints as opposed to having a penalty term) has been considered in~\cite{calder20} which show an ``ill-posedness result'' (the labels disappear in the large data limit) if the number of labelled points scales below a critical threshold, and a ``well-posedness result'' (the labels remain in the continuum problem) when the number of labels scales linearly with $n$.
Using graph $p$-Laplacian regularisation with finite labels the authors in~\cite{SlepcevThorpe,calder18} show that whether the variational problem is asymptotically well-posed depends on the choice of $p$ and how the graph is scaled.
For the fractional graph Laplacian with finite labels, it is shown in~\cite{dunlop18} that the problem is ill-posed if $s\leq\frac{d}{2}$ or the length scale on the graph is sufficiently large and conjectured that this is sharp. 

We wrap up this brief literature review by pointing out that other approaches to regression on unknown manifolds include \cite{Biometrika1}, where local tangent planes around points are carefully constructed to apply regression methods in the more classical functional data setting.
Our approach is markedly different as it does not rely on the construction of \textit{extrinsic} geometric objects.
In particular, once a proximity graph is defined on the data cloud, all regularisers and the resulting PDEs become \textit{intrinsic} to the graph.
\vspace{\baselineskip}

In the remainder of this section we define the graph and  continuum operators that are analysed in the paper, and then state our main results.

\subsection{Discrete Operators} \label{subsec:Intro:DisOp}

We begin by stating our basic assumptions on the data, and the graph that we use to model it:
\begin{enumerate}[label=(A{{\arabic*}})]
\item\label{ass:Intro:DisOp:A1}{\bf Assumptions on the domain $\Omega$\,:} $\Omega$ is a $d$-dimensional torus.
\item\label{ass:Intro:DisOp:A2} {\bf Assumptions on the feature vectors $x_i$\,:} $x_i\iid \mu$ where $\mu\in\cP(\Omega)$, where $\cP(\Omega)$ is the set of probability measures on $\Omega$;
\item\label{ass:Intro:DisOp:A3} {\bf Assumption on the density of $\mu$\,:} $\mu$ has a density $\rho\in\Cinfty(\Omega)$ that is bounded from above and below by positive constants, i.e. $0<\rho_{\min}:= \min_{x\in \Omega} \rho(x) \leq \max_{x\in\Omega} \rho(x) =: \rho_{\max} <+\infty$.
\item\label{ass:Intro:DisOp:A4} {\bf Assumptions on the graph constructed using the data $\{x_i\}$\,:} $G_n := (\Omega_n,W_n)$ where $\Omega_n=\{x_i\}_{i=1}^n$ are the nodes and $W_n=(W_{i,j})_{i,j=1}^n$ are the edge weights defined by $W_{i,j} = \eta_{\eps}(|x_i-x_j|)$ for $i\neq j$ and $W_{i,i}=0$.  Here $\eta_{\eps} = \frac{1}{\eps^d} \eta(\cdot/\eps)$ and where $\eta:[0,+\infty)\to [0,+\infty)$ is assumed to satisfy:
\begin{enumerate}
\item $\eta(t)>\frac12$ for all $t\leq \frac12$ and $\eta(t)=0$ for all $t\geq 1$; \label{ass:Intro:DisOp:A4:Spt}
\item $\eta$ is decreasing. \label{ass:Intro:DisOp:A4:Dec}
\end{enumerate}
\item\label{ass:Intro:DisOp:A5}{\bf Assumptions on the labelled data\,:} for each $i\in \bbN$, $y_i= g(x_i)+\xi_i$, for $g\in \Cinfty(\Omega)$ and $\xi_i\in \bbR$ are independent and identically distributed (iid), sub-Gaussian centred noise (where sub-Gaussian by definition means there exists $C>c>0$ such that $\bbP\l |\xi_j|>t\r\leq Ce^{-ct^2}$ for all $t\geq 0$). 
\end{enumerate}

\begin{remark}
The assumptions on $\eta$ are technical in nature and are imposed to facilitate some very concrete steps in our analysis.
Assumption \ref{ass:Intro:DisOp:A4:Dec} is slightly stronger than what is typically required in related papers, and will only be used to simplify our computations in, for example, Lemma \ref{lem:VarL2:Noise:OpBounds}.
\end{remark}

The graph Laplacian $\Delta_n$ plays an important role in the regularisation and is defined as follows:
\begin{equation} \label{eq:Intro:DisOp:Deltan}
\Delta_n := \frac{2}{n\eps^2}(D_n-W_n), \qquad D_n=(D_{i,j})_{i,j=1}^n \text{ diagonal matrix with } D_{i,i} = \sum_{k=1}^n W_{i,k}.
\end{equation}
Here we have chosen what is called the \emph{unnormalized} graph Laplacian.

Throughout the paper we will denote the empirical measure $\mu_n := \frac{1}{n} \sum_{i=1}^n \delta_{x_i}$.
We will define an inner product with respect to a (usually probability) measure $\nu$ by
\[ \langle u,v \rangle_{\Ltwo(\nu)} := \int_\Omega u(x) v(x) \, \dd \nu(x) \qquad \text{for } u,v \text{ measureable w.r.t. } \nu. \]
And the associated $\Ltwo$ norm by $\| u\|_{\Ltwo(\nu)} = \sqrt{\langle u,u\rangle_{\Ltwo(\nu)}}$.
When $\nu = \mu_n$ then the norm can be written $\| u\|_{\Ltwo(\mu_n)} = \sqrt{\frac{1}{n} \sum_{i=1}^n u(x_i)^2}$.

There is a small abuse in notation in how we define $\Delta_n$ since we will also write $\Delta_n:\Ltwo(\mu_n) \to \Ltwo(\mu_n)$; in this case we associate $u_n\in\Ltwo(\mu_n)$ with its vector representation $(u_n(x_1),\dots, u_n(x_n))^\top$.

Given $\mathbf{a}_n = (a_1,\dots, a_n),$ with $a_i \in \mathbb{R}$, we let 
\begin{equation} \label{eq:Intro:DisOp:Entauan}
\cEntau^{(\mathbf{a}_n)}= \frac{1}{n} \sum_{i=1}^n |u_n(x_i) - a_i|^2 + \tau R_n^{(s)}(u_n),
\end{equation}
where the regularisation is given by
\begin{equation} \label{eq:Intro:DisOp:Rns}
R_n^{(s)}(u_n) = \langle \Delta_n^s u_n,u_n\rangle_{\Ltwo(\mu_n)},
\end{equation}
and here $s$ is a positive integer with $\Delta_n^s$ the $s$-th power of the matrix.
We will mostly be concerned with the situation where $\mathbf{a}_n = \mathbf{y}_n = (y_1,\dots, y_n)$, which gives the energy
\begin{equation} \label{eq:Intro:DisOp:cEnyn}
\cEnyn(u_n) = \frac{1}{n} \sum_{i=1}^n |u_n(x_i) - y_i|^2 + \tau R_n^{(s)}(u_n)
\end{equation}
We will define $u_{n,\tau}^*$ to be the minimiser of~\eqref{eq:Intro:DisOp:cEnyn}. 
Note that when $s=1$, 
\[ R_n^{(1)}(u_n) = \frac{1}{n^2\eps^{2}} \sum_{i,j=1}^n W_{i,j} |u_n(x_i) - u_n(x_j)|^2 \]
and the regularisation functional is the graph Dirichlet energy.
We define $R_n^{(s)}$ for non-integer powers via the eigenvector-eigenvalue expansion (however our results consider only integer powers).
That is, we let $(\lambda_i^{(n)},q_i^{(n)})$ be eigenpairs of $\Delta_n$ then, since $\{q_i^{(n)}\}_{i=1}^n$ form an orthonormal basis of $\Ltwo(\mu_n)$ and we can write
\[ R_n^{(s)}(u_n) = \sum_{i=1}^n (\lambda_i^{(n)})^s \langle u_n,q_i^{(n)}\rangle_{\Ltwo(\mu_n)}^2 \]
which is defined for any $s\in\bbR$.

\subsection{Continuum Operators} \label{subsec:Intro:ContOp}

We now define the appropriate continuum operators and variational formulations.
It is well-known that as $n \to \infty$, the operator $\Delta_n$ converges to a continuum limit $\Delta_\rho$ \cite{belkin2007convergence,HeinvonLuxburgAudibert,Singer,calder20,calder18,GTSSpectralClustering,SpecRatesTrillos}, where $\Delta_\rho$ is the differential operator defined by
\begin{equation} \label{eq:Intro:ContOp:Delta}
\Delta_\rho \phi := -\frac{\sigma_\eta}{\rho} \Div(\rho^2\nabla \phi)
\end{equation}
and $\sigma_\eta$ is the constant defined by
\begin{equation} \label{eq:Intro:ContOp:sigmaeta}
\sigma_\eta:= \int_{\bbR^d} \eta(|h|) |h_1|^2 \, \dd h.
\end{equation}

For $\tau>0$ fixed, the continuum objective functional is defined by
\begin{equation} \label{eq:Intro:ContOp:cEinftyg}
\cEinftyg(u) = \int_\Omega |u(x) - g(x)|^2 \rho(x)\, \dd x + \tau R_\infty^{(s)}(u)
\end{equation}
where
\begin{equation} \label{eq:Intro:ContOp:Rinftys}
R_\infty^{(s)}(u) = \langle \Delta_\rho^s u,u\rangle_{\Ltwo(\mu)}.
\end{equation}
We will define $u_\tau^*$ to be the minimiser of~\eqref{eq:Intro:ContOp:cEinftyg}.
Again, we observe that when $s=1$ the regularisation functional,
\[ R_\infty^{(1)}(u) = \sigma_\eta \int_\Omega |\nabla u(x)|^2 \rho^2(x) \, \dd x, \]
is a weighted Dirichlet energy.

We remark that, by the fact that $\rho$ is bounded from below, we may integrate by parts to obtain
\[ cR_\infty^{(s)}(u) \leq  \int_\Omega |D^s u(x)|^2\,\dd x \leq C R_\infty^{(s)}(u) \]
for some constants $C>c>0$.

We can also define $R_\infty^{(s)}$ for non-integer powers analogously to the discrete case.
More concretely, by the spectral theorem and the fact that $\Omega$ is compact, if we let $(\lambda_i,q_i)$ be eigenpairs of $\Delta_\rho$ then $\{q_i\}_{i=1}^\infty$ form an orthonormal basis of $\Ltwo(\mu)$. 
In turn we can define
\[ R_\infty^{(s)}(u) = \sum_{i=1}^\infty \lambda_i^s \langle u,q_i\rangle_{\Ltwo(\mu)}^2 \]
which is well-defined for any $s\in \bbR$.

\subsection{Main Results} \label{subsec:Intro:Res}

Our results are to bound the bias and variance of the estimator $u_{n,\tau}^*$, defined as the minimiser of $\cEnyn$.
Following the terminology of \cite{cucker2002best} we define the variance of the estimator by
\[ \| u_{n,\tau}^* - u^*_\tau\lfloor_{\Omega_n} \|_{\Ltwo(\mu_n)} \]
where $u^*_\tau$ is the minimiser of $\cEinftyg$, and the bias is defined to be
\[ \| u^*_\tau - g \|_{\Ltwo(\mu)}. \]
The main results are the following.

\subsubsection{\texorpdfstring{$\Lp{2}$}{L2} Variance Estimates} \label{subsubsec:Intro:Res:VarL2}

We state the $\Lp{2}$ variance estimates in the following theorem.

\begin{theorem}[{\bf Variance Estimates}]
\label{thm:Intro:Res:VarL2:VarBound}
Let Assumptions~\ref{ass:Intro:DisOp:A1}-\ref{ass:Intro:DisOp:A5} hold and $s\in\bbN$.
We define $\cEnyn$ by~\eqref{eq:Intro:DisOp:cEnyn} and $\cEinftyg$ by~\eqref{eq:Intro:ContOp:cEinftyg} where $R_n^{(s)}$ is defined by~\eqref{eq:Intro:DisOp:Rns}, $R_\infty^{(s)}$ by~\eqref{eq:Intro:ContOp:Rinftys}, $\Delta_n$ by~\eqref{eq:Intro:DisOp:Deltan} and $\Delta_\rho$ by~\eqref{eq:Intro:ContOp:Delta}.
Let $u_{n,\tau}^*$ be the minimiser of $\cEnyn$ and $u^*_\tau$ be the minimiser of $\cEinftyg$.
Then, for all $\alpha>1$, there exists $\eps_0>0$, $\tau_0>0$, $C>c>0$ such that for all $\eps,n$ satisfying
\begin{equation} \label{eq:Intro:Res:VarL2:epsBound}
\eps_0 \geq \eps \geq C\sqrt[d]{\frac{\log(n)}{n}}
\end{equation}
and $\tau\in (0,\tau_0)$ we have % $\vartheta\in [\eps,1/\eps]$,
\begin{align*}
\| u_{n,\tau}^* - u^*_\tau\lfloor_{\Omega_n} \|_{\Ltwo(\mu_n)} & \leq C\l \sqrt{\frac{\log(n)}{n\eps^{d}}} + \frac{\eps^{2s}}{\tau} + \tau\eps \r 
\end{align*}
with probability at least $1-C\l n^{-\alpha} + n e^{-cn\eps^{d+4s}}\r$.
\end{theorem}

Let $\mathbf{g}_n = (g(x_1),\dots, g(x_n))$, and let $u_{n,\tau}^{g*}$ be the minimiser of the ``noiseless'' energy $\cEngn$.
The proof of Theorem~\ref{thm:Intro:Res:VarL2:VarBound} is divided into two steps; in the first we compare $u^*_{n,\tau}$ and $u^{g*}_{n,\tau}$ (corresponding to averaging in $y$) which gives a quantitative bound on the effect of the noise, in the second part we compare $u^{g*}_{n,\tau}$ with $u^*_\tau$ (which corresponds to averaging in $x$).
We do this in Sections~\ref{subsec:VarL2:Noise} and~\ref{subsec:VarL2:DisCtsNoiseless} respectively.

\begin{remark}
We notice that the estimates are meaningful for fixed $\tau$ when $n$ goes to infinity, i.e. $\tau$ is not required to become smaller with growing $n$.
\end{remark}

\begin{remark}
\label{rem:Intro:Res:VarL2:VarDeltaBound}
In addition to the bound in $\Ltwo(\mu_n)$ between $u_{n,\tau}^*$ and $u_\tau^*$ we are able to show a bound between the Laplacians $\Delta_n^{\frac{s}{2}} u_{n,\tau}^*$ and $\Delta_\rho^{\frac{s}{2}} u^*_\tau\lfloor_{\Omega_n}$ when $s$ is even.
More precisely, our results show,
\[ \lda \Delta_n^{\frac{s}{2}} u_{n,\tau}^* - \Delta_\rho^{\frac{s}{2}} u_\tau^*\lfloor_{\Omega_n} \rda_{\Ltwo(\mu_n)} \leq C \l \sqrt{\frac{\log(n)}{n\eps^d\tau}} + \frac{\eps^s}{\tau} + \eps \r \]
with the same probability as in Theorem~\ref{thm:Intro:Res:VarL2:VarBound}.
This inequality likely generalises to odd $s$, but to prove it using the methods in this paper we would require a pointwise convergence result for the graph derivative (i.e. the operator $\nabla_n$ such that $\Div_n \circ \nabla_n = \Delta_n$ and $\Div_n$ is the adjoint operator to $\nabla_n$) which is beyond the scope of the paper.
\end{remark}

\begin{remark}
\label{rem:Intro:Res:VarL2:CompHafiene}
We offer a comparison with the estimates in~\cite{hafiene-variational} (although note that a direct comparison is not possible as we scale $\eps\to0$ whilst~\cite{hafiene-variational} work in the setting where $\eps>0$ is fixed).
If, as in~\cite{hafiene-variational}, we fix $\tau>0$, and therefore absorb it into our constants, and choose $s=1$ then the error bound simplifies to
\[ \| u_{n,\tau}^* - u^*_\tau\lfloor_{\Omega_n} \|_{\Ltwo(\mu_n)} \leq C\l \sqrt{\frac{\log(n)}{n\eps^d}} + \eps \r. \]
Unfortunately, optimising over $\eps$ implies a scaling in $\eps=\eps_n$ of
\[ \eps_n \sim \l\frac{\log(n)}{n}\r^{\frac{1}{d+2}} \]
which is outside of the conditions of Theorem~\ref{thm:Intro:Res:VarL2:VarBound} as~\eqref{eq:Intro:Res:VarL2:epsBound} does not hold (one needs $n\eps^{d+4}\gg \log(n)$ in order to get a high probability bound).
Instead we choose $\eps_n\sim \l\frac{\log(n)}{n}\r^{\frac{1}{d+4}}$. 
With this choice the error scales as
\[ \| u_{n,\tau}^* - u^*_\tau\lfloor_{\Omega_n} \|_{\Ltwo(\mu_n)} \lesssim \l\frac{\log(n)}{n} \r^{\frac{1}{d+4}}. \]
The results in~\cite{hafiene-variational} show that for the $p$-Laplacian regularized problem a rate of convergence $n^{-\frac14}$ when $d=1$, $s=1$ and $\eps>0$ is fixed independently of $n$, compared to our rate of convergence of $n^{-\frac15}$ (up to logarithms).
\end{remark}

\subsubsection{\texorpdfstring{$\Lp{2}$}{L2} Bias Estimates} \label{subsubsec:Intro:Res:BiasL2}

We have the following $\Lp{2}$ bias estimate.

\begin{theorem}[{\bf Bias Estimates}]
\label{thm:Intro:Res:BiasL2:BiasBound}
Let Assumptions~\ref{ass:Intro:DisOp:A1},\ref{ass:Intro:DisOp:A3} hold and $\tau>0$, $s\geq 1$ and $g\in \SobSpace^s(\Omega)$. 
We define $\cEinftyg$ by~\eqref{eq:Intro:ContOp:cEinftyg} where $R_\infty^{(s)}$ is defined by~\eqref{eq:Intro:ContOp:Rinftys} and $\Delta_\rho$ by~\eqref{eq:Intro:ContOp:Delta}.
Let $u_\tau^*$ be the minimiser of $\cEinftyg$, then
\[ \| u_\tau^* - g \|_{\Ltwo(\mu)} \leq \tau \|\Delta_\rho^s g\|_{\Ltwo(\mu)}. \]
\end{theorem}

The theorem is proved in Section~\ref{sec:BiasL2}.

\begin{remark}
\label{rem:Intro:Res:BiasL2:BiasDeltaBound}
We are also able to show that
\[ \lda \Delta_\rho^{\frac{s}{2}} (g-u_\tau^*) \rda_{\Ltwo(\mu)} \leq \sqrt{\frac{\tau}{2}} \lda \Delta_\rho^s g\rda_{\Ltwo(\mu)}. \]
\end{remark}

\subsubsection{\texorpdfstring{$\Lp{2}$}{L2} Error Estimates} \label{subsubsec:Intro:Res:ErrorL2}

The previous results, from Sections~\ref{subsubsec:Intro:Res:VarL2} and~\ref{subsubsec:Intro:Res:BiasL2}, can be combined to bound the error between $u_{n,\tau}^*$ and $g$. 
Trivially, by the triangle inequality, we may bound the total error by
\[ \| u_{n,\tau}^* - g\lfloor_{\Omega_n} \|_{\Ltwo(\mu_n)} \leq \| u_{n,\tau}^* - u^*_\tau\lfloor_{\Omega_n} \|_{\Ltwo(\mu_n)} + \| u^*_\tau\lfloor_{\Omega_n} - g\lfloor_{\Omega_n} \|_{\Ltwo(\mu_n)}. \]
Introducing a transport map $T_n:\Omega_n\to\Omega$ satisfying $(T_n)_{\#}\mu = \mu_n$ we can write
\begin{align*}
\| u^*_\tau\lfloor_{\Omega_n} - g\lfloor_{\Omega_n} \|_{\Ltwo(\mu_n)} & = \| u^*_\tau\lfloor_{\Omega_n}\circ T_n - g\lfloor_{\Omega_n}\circ T_n \|_{\Ltwo(\mu)} \\
 & \leq \| u^*_\tau\lfloor_{\Omega_n} \circ T_n - u^*_\tau \|_{\Ltwo(\mu)} + \| u^*_\tau - g\|_{\Ltwo(\mu)} + \| g - g\lfloor_{\Omega_n}\circ T_n\|_{\Ltwo(\mu)}.
\end{align*}
Assuming $g$ is Lipschitz then we can bound $\| g - g\lfloor_{\Omega_n}\circ T_n\|_{\Ltwo(\mu)}\leq \Lip(g) \|T_n-\Id\|_{\Lany^{2}(\mu)}$.

By Lemma~\ref{lem:VarL2:DisCtsNoiseless:L2VarProof:utau*Bound} we have that $\{ u_\tau^* \}_{0<\tau<\tau_0}$ is bounded in $\Ck{1}(\Omega)$.
Hence, there exists $L>0$ such that $|u_\tau^*(x) - u_\tau^*(y)| \leq L|x-y|$ for all $x,y\in \Omega$ and $0<\tau<\tau_0$.
We can then bound $\| u^*_\tau\lfloor_{\Omega_n} \circ T_n - u^*_\tau \|_{\Ltwo(\mu)} \leq L\| T_n - \Id\|_{\Lany^{2}(\mu)}$.
Now, since we can choose any $T_n$ satisfying $(T_n)_{\#}\mu = \mu_n$ then we choose the one that minimises the $\Ltwo$ distance between $T_n$ and $\Id$, by~\cite{Fournier-Guillin} this can be bounded
\[ \| T_n - \Id\|_{\Ltwo(\mu)} \leq C\sqrt[d]{\frac{|\log(\delta)|}{n}} \] 
with probability at least $1-\delta$.

Putting the previous argument together, with the choice $\delta = n^{-\alpha}$, we have the following corollary.

\begin{corollary}
\label{cor:Intro:Res:ErrorL2:ErrorL2}
Let Assumptions~\ref{ass:Intro:DisOp:A1}-\ref{ass:Intro:DisOp:A5} hold and $s\in \bbN$.
We define $\cEnyn$ by~\eqref{eq:Intro:DisOp:cEnyn} where $R_n^{(s)}$ is defined by~\eqref{eq:Intro:DisOp:Rns} and $\Delta_n$ by~\eqref{eq:Intro:DisOp:Deltan}.
Let $u_{n,\tau}^*$ be the minimiser of $\cEnyn$.
Then, for all $\alpha>1$, there exists $\eps_0>0$, $C>c>0$ such that for all $\eps, n$ satisfying~\eqref{eq:Intro:Res:VarL2:epsBound} and $\tau\in (0,\tau_0)$ we have 
\[ \| u_{n,\tau}^* - g\lfloor_{\Omega_n} \|_{\Ltwo(\mu_n)} \leq C\l \sqrt{\frac{\log(n)}{n\eps^d}} + \frac{\eps^{2s}}{\tau} + \tau \r \]
with probability at least $1-C\l n^{-\alpha} + ne^{-cn\eps^{d+4s}}\r$.
\end{corollary}

\begin{remark}
\label{rem:Intro:Res:ErrorL2:DeltaErrorL2}
Combining Remarks~\ref{rem:Intro:Res:VarL2:VarDeltaBound} and~\ref{rem:Intro:Res:BiasL2:BiasDeltaBound} we can also show, for even $s$,
\[ \| \Delta_n^{\frac{s}{2}} u_{n,\tau}^* - \Delta_\rho^{\frac{s}{2}}g\lfloor_{\Omega_n} \|_{\Ltwo(\mu_n)} \leq C\l \sqrt{\frac{\log(n)}{\tau n\eps^d}} + \frac{\eps^{s}}{\tau} + \eps + \sqrt{\tau} \r \] %\frac{(\log(n))^{p_{\tilde{d}}}}{n^{\frac{1}{\tilde{d}}}} \r \]
with probability at least $1-C\l n^{-\alpha} + ne^{-cn\eps^{d+4s}}\r$.
\end{remark}

\begin{remark}
\label{rem:Intro:Res:ErrorL2:s=1}
When $s=1$ the error simplifies to 
\[ \| u_{n,\tau}^* - g\lfloor_{\Omega_n} \|_{\Ltwo(\mu_n)} \leq C\l \sqrt{\frac{\log(n)}{n\eps^d}} + \frac{\eps^{2}}{\tau} + \tau \r \]
with probability at least $1-C\l n^{-\alpha} + ne^{-cn\eps^{d+4}}\r$.
Choosing $\tau$ optimally with respect to $\eps$ implies $\tau=\eps$ and
\[ \| u_{n,\tau}^* - g\lfloor_{\Omega_n} \|_{\Ltwo(\mu_n)} \leq C\l \sqrt{\frac{\log(n)}{n\eps^d}} + \eps \r. \]
The optimal choice of $\eps$ is $\eps_n = \l\frac{\log(n)}{n}\r^{\frac{1}{d+2}}$, which (as in Remark~\ref{rem:Intro:Res:VarL2:CompHafiene}) is outside the admissible scaling of $\eps_n$, so we choose $\eps_n\sim \l\frac{\log(n)}{n}\r^{\frac{1}{d+4}}$.
In this regime the optimal error is
\[ \| u_{n,\tau}^* - g\lfloor_{\Omega_n} \|_{\Ltwo(\mu_n)} \leq C\l \frac{\log(n)}{n} \r^{\frac{1}{d+4}}. \]
This is approximately the minimax rate achieved for the total variation regularised problem which, in certain cases, is up to logarithms scaling as $n^{-\frac{1}{d}}$~\cite{padilla2019adaptive}, comparable to the $\Linfty$ minimax rates and convergence of spline smoothing obtained in \cite{Kpotufe}, \cite{GTM-Regression} and~\cite{stone82}, which are approximately $n^{-\frac{1}{d+2}}$.
This also coincides with the semi-supervised rate of convergence given in~\cite{calder20} when the number of labels are linear in $n$.
\end{remark}

\begin{remark}
\label{rem:Intro:Res:ErrorL2:Splines}
For $s\in \bbN$ we can choose $\tau = \eps^s$ so that the error simplifies to
\[ \| u_{n,\tau}^* - g\lfloor_{\Omega_n} \|_{\Ltwo(\mu_n)} \leq C\l \sqrt{\frac{\log(n)}{n\eps^d}} + \eps^s \r \]
with probability at least $1-C(n^{-\alpha} + ne^{-cn\eps^{d+4s}})$.
If we choose $\eps_n \sim \l\frac{M\log(n)}{n}\r^{\frac{1}{d+4s}}$ then error scales as
\[ \| u_{n,\tau}^* - g\lfloor_{\Omega_n} \|_{\Ltwo(\mu_n)} \leq C\l \frac{\log(n)}{n}\r^{\frac{s}{d+4s}} \]
(where $C$ depends on $M$) with probability at least $1-C\l n^{-\alpha}+ n^{1-cM}\r$, choosing $M=\frac{1+\alpha}{c}$ we have that the bound holds with probability at least $1-Cn^{-\alpha}$.
This is the close to the spline error rate which, up to logarithms, scales as $n^{-\frac{s}{d+2s}}$, see~\cite{stone82}.
\end{remark}

\begin{remark}
\label{rem:Intro:Res:ErrorL2:Tibshirani}
The minimax rates for estimating $g$ from noisy samples~\eqref{eq:Intro:Labels} is $n^{-\frac{2s}{2s+d}}$ when $g\in\Hs$ (the rate achieved by splines).
In the graph setting the minimax rate can be obtained by projecting the samples onto the first $K=K(\|g\|_{\Hs})$ eigenvectors of the graph Laplacian~\cite{green2021minimax}.
Whilst this has the advantage of better rates one must have an a-piori estimate in the $\Hs$ norm of $g$ in order to know $K$.
\end{remark}

\subsection{Outline} \label{subsec:Intro:Out}

The rest of the paper is organized as follows.
In Section~\ref{sec:VarL2} we obtain the $\Ltwo$ variance estimates discussed in Section~\ref{subsubsec:Intro:Res:VarL2}.
In Section~\ref{sec:BiasL2} we consider the bias of the estimation procedure given in Section~\ref{subsubsec:Intro:Res:BiasL2}.

\section{\texorpdfstring{$\Lp{2}$}{L2} Variance Estimates} \label{sec:VarL2}

In this section we prove the variance estimates stated precisely in Theorem \ref{thm:Intro:Res:VarL2:VarBound}.
We split the proof into two main steps.
First, we compare the solution $u_{n,\tau}^*$ with a discrete noiseless function $u_{n,\tau}^{g*}$.
Then, we compare the function $u_{n,\tau}^{g*}$ with $u_{\tau}^*$.

\subsection{Removing the Noise} \label{subsec:VarL2:Noise}

We start by stating the main result of the section.

\begin{proposition}
\label{prop:VarL2:Noise:NoiseDisEst}
Let Assumptions~\ref{ass:Intro:DisOp:A1}-\ref{ass:Intro:DisOp:A4} hold and $s\in \bbN$.
Let $\cEnyn$ be defined by~\eqref{eq:Intro:DisOp:Entauan}, where $R^{(s)}_n$, $\Delta_n$ are defined by~\eqref{eq:Intro:DisOp:Rns}, \eqref{eq:Intro:DisOp:Deltan} respectively, and let $u_{n,\tau}^*$, $u_{n,\tau}^{g*}$ be the minimisers of $\cEnyn$, $\cEngn$ respectively.
Assume that $\xi_i$ are iid, mean zero, sub-Gaussian random variables.
Then, for all $\alpha>1$, there exists $\eps_0>0$ and $C>0$ such that for any $\eps, n$ satisfying~\eqref{eq:Intro:Res:VarL2:epsBound} and $\tau>0$ we have
\begin{align*}
\lda u_{n,\tau}^* - u_{n,\tau}^{g*} \rda_{\Ltwo(\mu_n)} & \leq C \l \sqrt{\frac{\log(n)}{n\eps^d}} + \frac{\eps^{2s}}{\tau} \r \\
\lda \Delta_n^{\frac{s}{2}} u_{n,\tau}^* - \Delta_n^{\frac{s}{2}} u_{n,\tau}^{g*} \rda_{\Ltwo(\mu_n)} & \leq C\l \sqrt{\frac{\log(n)}{n\eps^d\tau}} + \frac{\eps^{s}}{\tau}\r
\end{align*}
with probability at least $1-C n^{-\alpha}$. 
\end{proposition}

The proof of the proposition will be given at the end of the section.
The strategy is to compare the Euler--Lagrange equations associated with minimising $\cEnyn$ and $\cEngn$.
In particular, we have
\[ \frac12 \nabla_{\Ltwo(\mu_n)} \cEnyn(u_n) = \tau \Delta_n^s u_n + u_n - \mathbf{y}_n \]
and therefore
\begin{align}
\tau \Delta_n^s u_{n,\tau}^* + u_{n,\tau}^* - \mathbf{y}_n & = 0 \notag \\
\tau \Delta_n^s u_{n,\tau}^{g*} + u_{n,\tau}^{g*} - \mathbf{g}_n & = 0. \label{eq:VarL2:Noise:ELNoiseless}
\end{align}
We let $w_{n,\tau}^* = u_{n,\tau}^* - u_{n,\tau}^{g*}$ then it follows that
\[ \tau \Delta_n^s w_{n,\tau}^* + w_{n,\tau}^* - (\mathbf{y}_n-\mathbf{g}_n) = 0 \]
and $w_{n,\tau}^*$ minimises $\cEn^{(\mathbf{y}_n-\mathbf{g}_n)} = \cEn^{(\boldsymbol{\xi}_n)}$ where $\boldsymbol{\xi}_n = (\xi_1,\dots, \xi_n)$.
We can write
\begin{equation} \label{eq:VarL2:Noise:wn*}
w_{n,\tau}^* = \l \tau \Delta_n^s + \Id\r^{-1} \boldsymbol{\xi}_n.
\end{equation}
To obtain an estimate on $\|w_{n,\tau}^*\|_{\Ltwo(\mu_n)}$ we use an ansatz $\tilde{w}_n$ and show $\|w_{n,\tau}^* - \tilde{w}_n\|_{\Ltwo(\mu_n)}\leq C \sqrt{\frac{\log(n)}{n\eps^d}}$ and $\|\tilde{w}_n\|_{\Ltwo(\mu_n)}\leq \frac{C\eps^{2s}}{\tau}$ with high probability.
Our choice of ansatz is to assume that the diagonal part of $\Delta_n$ dominates and therefore $\Delta_n\approx \frac{2}{n\eps^2} D_n$ which leads to the choice,
\begin{equation} \label{eq:VarL2:Noise:wtilden}
\tilde{w}_n = \l \tau \l \frac{2}{n\eps_n^2} D_n\r^s + \Id\r^{-1} \boldsymbol{\xi}_n. 
\end{equation}
This choice of ansatz is appropriate because the off-diagonal elements of $\Delta_n$ equate to a local averaging procedure, which with high probability will not amplify the vector $\xi$.
We can equivalently write
\begin{equation} \label{eq:VarL2:Noise:wtilden-vec}
\tilde{w}_n(x_i) = \frac{\xi_i}{\tau\l \frac{2}{n\eps^2} \sum_{k=1}^n W_{i,k}\r^s + 1}.
\end{equation}

The following lemmas will be useful.

\begin{lemma}
\label{lem:VarL2:Noise:DegreeBounds}
Under Assumptions~\ref{ass:Intro:DisOp:A1}-\ref{ass:Intro:DisOp:A4} there exists $C,C_1,C_2,c>0$ such that, if $n\eps^d\geq 1$ then
\[ C_1 \leq \frac{1}{n} \sum_{j=1}^n W_{i,j} \leq C_2 \qquad \text{and} \qquad \#\lb j\,:\, W_{i,j}>0\rb \leq Cn\eps^d \]
for all $i=1,\dots, n$ with probability at least $1-2ne^{-cn\eps^d}$.
\end{lemma}

\begin{proof}
Fix $i\in \bbN$, then $W_{i,j}$ are iid for $j\neq i$.
If $M=\|\eta_\eps\|_{\Lp{\infty}(\bbR)}$ and $\sigma^2 = \bbE\l W_{i,j} - \bbE[W_{i,j}]\r^2$ (where the expectation $\bbE[W_{i,j}]$ is taken over $x_j$) then it is straightforward to show the bounds $\sigma^2\leq C\eps^{-d}$ and $M\leq C\eps^{-d}$.
By Bernstein's inequality, for all $t>0$,
\[ \bbP\l \la \sum_{j\neq i} \l W_{i,j} - \bbE[W_{i,j}]\r \ra >t \r \leq 2\exp\l -\frac{t^2}{2n\sigma^2+\frac{4Mt}{3}}\r \leq 2\exp\l -\frac{ct^2\eps^d}{n + t} \r. \]
Choosing $t=\lambda n$ and restricting to $\lambda \leq 1$ we have
\[ \bbP\l \la \sum_{j\neq i} \l W_{i,j} - \bbE[W_{i,j}]\r \ra >\lambda n \r \leq 2\exp\l -cn\eps^d\lambda^2 \r. \]
Hence, (recalling $W_{i,i}=0$)
\[ (n-1) \bbE[W_{i,j}] - \lambda n \leq \sum_{j=1}^n W_{i,j} \leq (n-1) \bbE[W_{i,j}] + \lambda n \]
with probability at least $1-2e^{-cn\eps^d\lambda^2 }$.
One can show that there exists $C_1,C_2$ such that $C_1\leq \bbE[W_{i,j}] \leq C_2$.
For $n\geq 2$ (so that $n-1\geq n/2$),
\[ \frac{C_1}{2} - \lambda \leq \frac{1}{n}  \sum_{j=1}^n W_{i,j} \leq C_2 + \lambda \]
with probability at least $1-2e^{-cn\eps^d\lambda^2 }$.
Choosing $\lambda = \frac{C_1}{4}$ and union bounding over $i\in\{1,\dots,n\}$ we can conclude the first result.
	
The second result follows from the first by choosing $\tilde{\eps}=2\eps$ and letting $\tilde{W}_{i,j} = \eta_{\tilde{\eps}}(|x_i-x_j|)$.
Then, $W_{i,j}>0$ implies $\tilde{\eps}^d \tilde{W}_{ij}\geq 0.5$ and therefore $\#\{j\,:\,W_{i,j}>0\} \leq 2\tilde{\eps}^d \sum_{j=1}^n \tilde{W}_{i,j}$. 
Applying the first part of the lemma we have $2\tilde{\eps}^d \sum_{j=1}^n \tilde{W}_{i,j}\leq 2C_2n\tilde{\eps}^d = 2^{d+1} C_2n\eps^d$ as required.
\end{proof}

\begin{lemma}
\label{lem:VarL2:Noise:OpBounds}
Under Assumptions~\ref{ass:Intro:DisOp:A1}-\ref{ass:Intro:DisOp:A4} define $\Delta_n$ and $D_n$ by~\eqref{eq:Intro:DisOp:Deltan}.
Then, for all $\alpha>1$, there exists $\eps_0>0$ and $C>0$ such that for all $\eps, n$ satisfying~\eqref{eq:Intro:Res:VarL2:epsBound} we have
\[ \|\Delta_n\|_{\op}\leq \frac{C}{\eps^2}, \qquad  \|D_n \|_{\op} \leq Cn, \qquad \text{and} \qquad \|W_n\|_{\op} \leq Cn \]
with probability at least $1-Cn^{-\alpha}$.
\end{lemma}

\begin{proof}
For $d\geq 3$ one can bound $\|\Delta_n\|_{\op}\leq C\eps^{-2}$ by~\cite[Lemma 22]{dunlop18}.
Indeed, $\|\Delta_n\|_{\op}\leq C\eps^{-2}$ whenever $\dWinfty(\mu_n,\mu)<\eps$, where $\dWinfty$ is the $\infty$-Wasserstein distance.
By~\cite[Theorem 1.1]{W8L8canadian} this holds with probability at least $1-Cn^{-\alpha}$. 
The same argument is used in~\eqref{eq:VarL2:Noise:DunlopCalc} below as one step in the proof for $d=2$. 
	
For $d=2$ a small modification is required to remove the additional logarithmic factors that are present in the scaling of $\dWinfty(\mu_n,\mu)$, i.e. one has $\dWinfty(\mu_n,\mu)\sim \frac{(\log(n))^{\frac34}}{\sqrt{n}}$ and therefore requires $\eps\geq C\frac{(\log(n))^{\frac34}}{\sqrt{n}}$.
However, this can be avoided by comparing $\mu_n$ to a smooth approximation of $\mu$.

In~\cite[Lemma 3.1]{caroccia19} the authors show, in the Euclidean setting, that if $\eps_n\to 0$ satisfies $\frac{\log n}{n\eps_n^d}\to 0$ then there exists an absolutely continuous probability measure $\tilde{\mu}_n\in\cP(\Omega)$ such that
\[ \frac{\dWinfty(\mu_n,\tilde{\mu}_n)}{\eps_n} \to 0 \qquad \text{and} \qquad \|\rho - \tilde{\rho}_n \|_{\Linfty(\mu)} \to 0 \]
where $\tilde{\rho}_n$ is the density of $\tilde{\mu}_n$.
As in~\cite[Proposition 2.10]{calder19} the proof can be modified to give a non-asymptotic quantitative high probability bound.
In particular, there exists constants $C$, $\eps_0$ and $\theta_0$ such that if $n^{-\frac{1}{d}} \leq \eps \leq \eps_0$ and $\theta\leq \theta_0$ then
\[ \dWinfty(\mu_n,\tilde{\mu}_n) \leq \eps \qquad \text{and} \qquad \|\rho - \tilde{\rho}_n \|_{\Linfty(\mu)} \leq C(\theta+\eps) \]
with probability at least $1-2ne^{-cn\theta^2\eps^d}$.
For the rest of the proof we fix $\theta=\theta_0$ and absorb it into the constants.
	
Note that if we define $\bar{\eta} = \eta((|\cdot|-1)_+)$ and $T_n:\Omega\to\Omega$ satisfies $\|T_n - \Id\|_{\Linfty(\Omega)}\leq\eps$ then
\[ \eta\l\frac{|x-T_n(y)|}{\eps}\r \leq \eta\l \frac{\l|x-y|-|T_n(y)-y|\r_+}{\eps}\r \leq \eta\l\l\la\frac{x-y}{\eps}\ra -1\r_+\r = \bar{\eta}\l\frac{|x-y|}{\eps}\r. \]
We choose $T_n$ to be a transport map satisfying $T_{n\#}\tilde{\mu}_n = \mu_n$ and $\|T_n-\Id\|_{\Linfty(\tilde{\mu}_n)} = \dWinfty(\mu_n,\tilde{\mu}_n)$.
	
Let $\lambda_{\max}$ be the largest eigenvalue of $\Delta_n$ then, as in the proof of~\cite[Lemma 22]{dunlop18},
\begin{align} \label{eq:VarL2:Noise:DunlopCalc}
\begin{split}
\lambda_{\max} & = \sup_{\|u\|_{\Ltwo(\mu_n)}=1} \langle u,\Delta_n u\rangle_{\Ltwo(\mu_n)} \\
 & \leq \frac{4}{n^2\eps^{d+2}} \sup_{\|u\|_{\Ltwo(\mu_n)}=1} \sum_{i,j=1}^n \eta\l\frac{|x_i-x_j|}{\eps}\r u(x_i)^2 \\
 & = \frac{4}{n\eps^{d+2}} \sup_{\|u\|_{\Ltwo(\mu_n)}=1} \sum_{i=1}^n u(x_i)^2 \int_\Omega \eta\l\frac{|x_i-T_n(y)|}{\eps}\r \tilde{\rho}_n(y) \, \dd y \\
 & \leq \frac{4}{n\eps^{d+2}} \sup_{\|u\|_{\Ltwo(\mu_n)}=1} \sum_{i=1}^n u(x_i)^2 \int_\Omega \bar{\eta}\l\frac{|x_i-y|}{\eps}\r \tilde{\rho}_n(y) \, \dd y \\
 & \leq \frac{4}{n\eps^{d+2}} \sup_{\|u\|_{\Ltwo(\mu_n)}=1} \sum_{i=1}^n u(x_i)^2 \int_\Omega \bar{\eta}\l\frac{|x_i-y|}{\eps}\r \l\rho(y) + C\r \, \dd y \\
 & \leq \frac{C}{\eps^2}
\end{split}
\end{align}
since $\int_{\bbR^d} \bar{\eta}(|z|) \, \dd z <+\infty$.
	
Although the bound holds for probability at least $1-Cne^{-cn\eps^d}$ when $d=2$ we can assume that the $C$ in~\eqref{eq:Intro:Res:VarL2:epsBound} is sufficiently large so that $\frac{n\eps^d}{\log(n)}\geq \frac{\alpha +1}{c}$.
After some elementary algebra one has that $1-Cne^{-cn\eps^d}\geq 1-Cn^{-\alpha}$.

For any $v\in \Ltwo(\mu_n)$ we have, by Lemma~\ref{lem:VarL2:Noise:DegreeBounds} with probability at least $1-2ne^{-cn\eps^d}$,
\[ \| D_nv\|_{\Ltwo(\mu_n)}^2 = \frac{1}{n} \sum_{i=1}^n \l v(x_i) \sum_{j=1}^n W_{i,j}\r^2 \leq C_2^2 n \sum_{i=1}^n v(x_i)^2 = C_2^2 n^2 \| v\|_{\Ltwo(\mu_n)}^2 \]
which implies $\|D_n\|_{\op}\leq C_2n$.
The operator norm bound on $W_n$ follows from the bounds on the operator norms of $\Delta_n$, $D_n$ and the triangle inequality.
Choosing $C$ in Equation~\eqref{eq:Intro:Res:VarL2:epsBound} sufficiently large we can assume that $1-2ne^{-cn\eps^d}\geq 1-Cn^{-\alpha}$.
\end{proof}

In fact,~\cite[Lemma 22]{dunlop18}, suggests the operator bound on $\Delta_n$ is sharp (up to a constant), that is, there exists $c>0$ such that $\|\Delta_n \|_{\op}\geq \frac{c}{\eps^2}$.

\begin{lemma}
\label{lem:VarL2:Noise:WnDnwtildeBound}
Let Assumptions~\ref{ass:Intro:DisOp:A1}-\ref{ass:Intro:DisOp:A4} hold and $s\geq 1$, $k\in \bbN$.
Let $\xi_i$ be iid, mean zero, sub-Gaussian random variables.
Define $\tilde{w}_n$ by~\eqref{eq:VarL2:Noise:wtilden-vec} and $D_n$ by~\eqref{eq:Intro:DisOp:Deltan}. 
Then, for any $\alpha>1$ there exists $\eps_0>0$ and $C>0$ such that for all $\eps, n$ satisfying~\eqref{eq:Intro:Res:VarL2:epsBound} and $\tau>0$ we have
\[ \| W_n (D_n)^{k-1} \tilde{w}_n \|_{\Ltwo(\mu_n)} \leq  \frac{Cn^k \eps^{2s}}{\tau} \sqrt{\frac{\log(n)}{n\eps^d}} \]
with probability at least $1-Cn^{-\alpha}$.
\end{lemma}

\begin{proof}
Let us condition on a graph $G_n$ that satisfies the two inequalities:
\begin{enumerate}
\item[(i)] $C_1\leq \frac{1}{n} \sum_{j=1}^n W_{i,j} \leq C_2$, for all $i=1,2,\dots,n$,
\item[(ii)] $\#\lb j\,:\, W_{i,j}>0\rb \leq Cn\eps^d$ for all $i=1,2,\dots,n$.
\end{enumerate}
Fix $i\in\{1,2,\dots,n\}$ and define
\[ q_j = \frac{\tau W_{i,j}}{\eps^{2s} n^{k-1}}[D_n^{k-1} \tilde{w}_n]_j. \]
Conditioned on $G_n$ we have that $q_j$ are zero mean and independent random variables.
Moreover, since
\[ q_j = \frac{\tau W_{i,j} \l \sum_{\ell=1}^n W_{j,\ell}\r^{k-1} \xi_j}{\eps^{2s} n^{k-1} \l \tau \l\frac{2}{n\eps^2} \sum_{\ell=1}^n W_{j,\ell}\r^s+1\r} \]
then we have $|q_j|\leq \frac{C|\xi_j|}{\eps^d}$ so $q_j$ is sub-Gaussian and $\|q_j\|_{\psi_2}\leq \frac{C\|\xi_j\|_{\psi_2}}{\eps^d}\lesssim \frac{1}{\eps^d}$ where $\|\cdot\|_{\psi_2}$ is the Birnbaum--Orlicz norm. 
By Hoeffding's inequality, for any $t>0$
\[ \bbP\l \la \sum_{j=1}^n q_j\ra>t | G_n\r \leq \bbP\l \la\sum_{j\,:\, W_{i,j}>0} q_j \ra>t | G_n\r \leq 2\exp\l-\frac{ct^2}{\sum_{j\,:\,W_{i,j}>0} \|q_j\|_{\psi_2}^2}\r \leq 2e^{-\frac{ct^2\eps^d}{n}}. \]
We choose $t = \lambda \sqrt{\frac{n\log(n)}{\eps^d}}$ so
\[ \frac{\tau}{\eps^{2s}n^{k-1}} \la \ls W_nD_n^{k-1}\tilde{w}_n\rs_i \ra = \la \sum_{j=1}^n q_j\ra \leq \lambda\sqrt{\frac{n\log(n)}{\eps^d}} \]
with probability at least $1-2n^{-c\lambda^2}$, conditioned on $G_n$.
Union bounding and selecting $\lambda = \sqrt{\frac{\alpha+1}{c}}$, we then get that the above bound holds for all $i\in\{1,2,\dots,n\}$ with probability at least $n^{1-c\lambda^2} = n^{-\alpha}$.
Hence, after absorbing $\alpha$ into the constant $C$,
\[ \lda W_nD_n^{k-1}\tilde{w}_n\rda_{\Ltwo(\mu_n)} \leq \lda W_nD_n^{k-1}\tilde{w}_n\rda_{\Linfty(\mu_n)} \leq \frac{Cn^k\eps^{2s}}{\tau} \sqrt{\frac{\log(n)}{n\eps^d}} \]
conditioned on $G_n$ with probability at least $1-2n^{-\alpha}$.
Since, by Lemma~\ref{lem:VarL2:Noise:DegreeBounds}, the probability of $G_n$ satisfying conditions (i) and (ii) is at least $1-2e^{-cn\eps^d}$, and  by choosing $C$ sufficiently large we have that $2e^{-cn\eps^d}\leq Cn^{-\alpha}$ we can conclude the lemma.
\end{proof}

To control $\|\tilde{w}_n-w_{n,\tau}^*\|_{\Ltwo(\mu_n)}$ we take advantage of the convexity of our objective functional $\cE_{n,\tau}^{(\boldsymbol{xi}_n)}$, where we recall that $\boldsymbol{\xi}_n = \mathbf{y}_n - \mathbf{g}_n$.
In particular, one can easily show that $\cE_{n,\tau}^{(\boldsymbol{\xi}_n)}$ satisfies
\[ \langle \nabla_{\Ltwo(\mu_n)} \cE_{n,\tau}^{(\boldsymbol{\xi}_n)}(v_n) - \nabla_{\Ltwo(\mu_n)} \cE_{n,\tau}^{(\boldsymbol{\xi}_n)}(u_n),v_n-u_n\rangle_{\Ltwo(\mu_n)} = 2\tau \| \Delta_n^{\frac{s}{2}} (v_n-u_n) \|_{\Ltwo(\mu_n)}^2 + 2 \| u_n-v_n\|_{\Ltwo(\mu_n)}^2 \]
for any $u_n,v_n\in \Ltwo(\mu_n)$.
Hence,
\begin{align*}
\| u_n - v_n \|_{\Ltwo(\mu_n)} & \leq \frac12 \| \nabla_{\Ltwo(\mu_n)} \cE_n^{(\boldsymbol{\xi}_n)}(v_n) - \nabla_{\Ltwo(\mu_n)} \cE_n^{(\boldsymbol{\xi}_n)}(u_n)\|_{\Ltwo(\mu_n)} \\
\| \Delta_n^{\frac{s}{2}}(u_n - v_n) \|_{\Ltwo(\mu_n)} & \leq \frac{1}{2\sqrt{\tau}} \| \nabla_{\Ltwo(\mu_n)} \cE_n^{(\boldsymbol{\xi}_n)}(v_n) - \nabla_{\Ltwo(\mu_n)} \cE_n^{(\boldsymbol{\xi}_n)}(u_n)\|_{\Ltwo(\mu_n)}.
\end{align*}
Applying this bound to $u_n=w_{n,\tau}^*$ and $v_n=\tilde{w}_n$, and using the optimality of $w_{n,\tau}^*$, we have
\begin{align}
\| w_{n,\tau}^* - \tilde{w}_n \|_{\Ltwo(\mu_n)} & \leq \frac12 \| \nabla_{\Ltwo(\mu_n)} \cE_n^{(\boldsymbol{\xi}_n)}(\tilde{w}_n)\|_{\Ltwo(\mu_n)} \label{eq:VarL2:Noise:wDiff} \\
\| \Delta_n^{\frac{s}{2}}(w_{n,\tau}^* - \tilde{w}_n) \|_{\Ltwo(\mu_n)} & \leq \frac{1}{2\sqrt{\tau}} \| \nabla_{\Ltwo(\mu_n)} \cE_n^{(\boldsymbol{\xi}_n)}(\tilde{w}_n)\|_{\Ltwo(\mu_n)}. \label{eq:VarL2:Noise:DeltanswDiffeq}
\end{align} 

The next lemma will bound these gradients in order to prove $\Lp{2}$ convergence rates.

\begin{lemma}
\label{lem:VarL2:Noise:wdiff}
Let Assumptions~\ref{ass:Intro:DisOp:A1}-\ref{ass:Intro:DisOp:A4} hold and $s\in \bbN$.
Let $\xi_i$ be iid, mean zero, sub-Gaussian, random variables.
Define  $\tilde{w}_n$ by~\eqref{eq:VarL2:Noise:wtilden-vec} and $w_{n,\tau}^*$ by~\eqref{eq:VarL2:Noise:wn*} where $\Delta_n$ is given by~\eqref{eq:Intro:DisOp:Deltan}.
Then, for any $\alpha>1$, there exists $C>0$ such that if $\eps, n$ satisfy~\eqref{eq:Intro:Res:VarL2:epsBound} and $\tau>0$ we have
\begin{align*}
\| \tilde{w}_n - w_{n,\tau}^* \|_{\Ltwo(\mu_n)} \leq C\sqrt{\frac{\log(n)}{n\eps^d}} \\
\| \Delta_n^{\frac{s}{2}} \tilde{w}_n - \Delta_n^{\frac{s}{2}}w_{n,\tau}^* \|_{\Ltwo(\mu_n)} \leq C\sqrt{\frac{\log(n)}{n\eps^d\tau}} 
\end{align*}
with probability at least $1-Cn^{-\alpha}$.
\end{lemma}

\begin{proof}
By the definition of $\tilde{w}_n$, namely Equation~\eqref{eq:VarL2:Noise:wtilden}
\begin{align*}
\frac12 \nabla_{\Ltwo(\mu_n)} \cE_{n,\tau}^{(\boldsymbol{\xi}_n)}(\tilde{w}_n) & = \l \tau \Delta_n^s +\Id \r \tilde{w}_n - \boldsymbol{\xi}_n \\
 & = \frac{2^s\tau}{n^s\eps^{2s}} \l (D_n-W_n)^s - D_n^s\r \tilde{w}_n \\
 & = \frac{2^s\tau}{n^s\eps^{2s}} \l \l \sum_{\chi \in \{0,1\}^s}\prod_{i=1}^s D_n^{\chi_i} (-W_n)^{1-\chi_i}\r -D_n^s  \r\tilde w_n.
\end{align*}
Using the bounds from Lemmas~\ref{lem:VarL2:Noise:OpBounds} and~\ref{lem:VarL2:Noise:WnDnwtildeBound} and their associated probability estimates, along with the fact that we have cancelled the $D_n^s$ term, we then may bound
\[ \frac12 \lda \nabla_{\Ltwo(\mu_n)} \cE_{n,\tau}^{(\boldsymbol{\xi}_n)}(\tilde{w}_n) \rda_{\Ltwo(\mu_n)} \leq  \frac{2^s\tau}{n^s\eps^{2s}} \sum_{\chi \in \{0,1\}^s} \frac{Cn^s \eps^{2s}}{\tau} \sqrt{\frac{\log(n)}{n \eps^d}} \leq C \sqrt{\frac{\log(n)}{n \eps^d}}. \]
This concludes the proof.
\end{proof}

Our final lemma before proving Proposition~\ref{prop:VarL2:Noise:NoiseDisEst} is to bound $\|\tilde{w}_n\|_{\Ltwo(\mu_n)}^2$.

\begin{lemma}
\label{lem:VarL2:Noise:wtilde}
Let Assumptions~\ref{ass:Intro:DisOp:A1}-\ref{ass:Intro:DisOp:A4} hold and $s\in \bbN$.
Let $\xi_i$ be iid, mean zero, sub-Gaussian, random variables.
Define  $\tilde{w}_n$ by~\eqref{eq:VarL2:Noise:wtilden-vec}.
Then, for any $\alpha>1$, there exists $C>0$ such that for all $\eps, n$ satisfying $n\eps^d\geq 1$ and $\tau>0$ we have
\[ \| \tilde{w}_n \|_{\Ltwo(\mu_n)} \leq \frac{C\eps^{2s}}{\tau} \]
with probability at least $1-n^{-\alpha}$.
\end{lemma}

\begin{proof}
By application of Lemma~\ref{lem:VarL2:Noise:DegreeBounds} we have
\begin{align*}
\|\tilde{w}_n \|_{\Ltwo(\mu_n)}^2 & = \frac{1}{n} \sum_{i=1}^n \frac{\xi_i^2}{(\tau(\frac{2}{n\eps^2} \sum_{k=1}^n W_{i,k})^s+1)^2} \\
 & \leq \frac{\eps^{4s}}{\tau^2 n} \sum_{i=1}^n \xi_i^2.
\end{align*}
Applying a Chernoff bound we have, for all $s,t\geq 0$,
\[ \bbP\l \sum_{i=1}^n \xi_i^2 \geq t \r \leq \frac{\bbE\ls e^{s\sum_{i=1}^n \xi_i^2}\rs}{e^{st}} = \frac{\prod_{i=1}^n \bbE\ls e^{s\xi_i^2}\rs}{e^{st}}. \]
Choosing $s=\|\xi_i\|_{\Psi_2}^{-2}$ and $t=An$ we have
\[ \bbP\l \frac{1}{n} \sum_{i=1}^n \xi_i^2 \geq A\r \leq 2^n e^{-An\|\xi_i\|_{\Psi_2}^{-2}}. \]
Now we choose $A$ sufficiently large so that $\frac{A}{\|\xi_i\|_{\Psi_2}^2} \geq \alpha + \log 2$ and hence
\[ 2^n e^{-An\|\xi_i\|_{\Psi_2}^{-2}} \leq e^{-n\alpha} \leq n^{-\alpha}. \]
In particular, $\|\tilde{w}_n \|_{\Ltwo(\mu_n)}\leq \frac{C\eps^{2s}}{\tau}$ with probability at least $1-n^{-\alpha}$ as required.
\end{proof}

The proof of Proposition~\ref{prop:VarL2:Noise:NoiseDisEst} now follows from Lemma~\ref{lem:VarL2:Noise:OpBounds}, Lemma~\ref{lem:VarL2:Noise:wdiff} and Lemma~\ref{lem:VarL2:Noise:wtilde} since
\[ \| u_{n,\tau}^* - u_{n,\tau}^{g*}\|_{\Ltwo(\mu_n)} \leq \| w_{n,\tau}^* - \tilde{w}_n \|_{\Ltwo(\mu_n)} + \|\tilde{w}_n\|_{\Ltwo(\mu_n)} \leq C\l \sqrt{\frac{\log(n)}{n\eps^d}} + \frac{\eps^{2s}}{\tau}\r \]
and
\begin{align*}
\lda \Delta_n^{\frac{s}{2}} u_{n,\tau}^* - \Delta_n^{\frac{s}{2}} u_{n,\tau}^{g*}\rda_{\Ltwo(\mu_n)} & \leq \| \Delta_n^{\frac{s}{2}} w_{n,\tau}^* - \Delta_n^{\frac{s}{2}} \tilde{w}_n \|_{\Ltwo(\mu_n)} + \|\Delta_n^{\frac{s}{2}} \tilde{w}_n\|_{\Ltwo(\mu_n)} \\
 & \leq \| \Delta_n^{\frac{s}{2}} w_{n,\tau}^* - \Delta_n^{\frac{s}{2}} \tilde{w}_n \|_{\Ltwo(\mu_n)} + \|\Delta_n\|_{\op}^{\frac{s}{2}} \|\tilde{w}_n\|_{\Ltwo(\mu_n)} \\
 & \leq C\l \sqrt{\frac{\log(n)}{n\eps^d\tau}} + \frac{\eps^{s}}{\tau}\r
\end{align*}
with probability at least $1-Cn^{-\alpha}$.

\subsection{Discrete-to-Continuum in the Noiseless Case} \label{subsec:VarL2:DisCtsNoiseless}

In this subsection we prove the following estimates which relate the functions $u_{n,\tau}^*$ (the minimizer of $\cEnyn$ defined in~\eqref{eq:Intro:DisOp:cEnyn}) with the function $u_{\tau}^*$ (the minimizer of $\cEinftyg$ defined in~\eqref{eq:Intro:ContOp:cEinftyg}).

As in~\eqref{eq:VarL2:Noise:ELNoiseless} we can write the Euler-Lagrange equations associated with minimizing $\cEinftyg$ by
\begin{equation} \label{eq:VarL2:DisCtsNoiseless:ELcEinftytau}
\tau \Delta_{\rho}^s u_\tau^* + u_\tau^* - g = 0.
\end{equation}
Our main result for this section is the following proposition.

\begin{proposition}
\label{prop:VarL2:DisCtsNoiseless:NoiselessLim}
Let Assumptions~\ref{ass:Intro:DisOp:A1}-\ref{ass:Intro:DisOp:A5} hold and $s\in \bbN$.
Define $\Delta_n$, $\Delta_\rho$ and $\sigma_\eta$ by~\eqref{eq:Intro:DisOp:Deltan}, \eqref{eq:Intro:ContOp:Delta} and~\eqref{eq:Intro:ContOp:sigmaeta} respectively.
Let $u_{n,\tau}^{g*}$ and $u^*_\tau$ satisfy~\eqref{eq:VarL2:Noise:ELNoiseless} and~\eqref{eq:VarL2:DisCtsNoiseless:ELcEinftytau} respectively.
Then, for any $\alpha>1$ and $\tau_0>0$ there exists constants $\eps_0>0$, $C>c>0$ such that, for any $\eps, n$ satisfying~\eqref{eq:Intro:Res:VarL2:epsBound} and $\tau\in (0,\tau_0)$ we have
\[ \lda u_{n,\tau}^{g^*} - u_{\tau}^{*}\lfloor_{\Omega_n}\rda_{\Lp{2}(\mu_n)} \leq C\tau\eps \qquad \lda \Delta_n^{\frac{s}{2}} u_{n,\tau}^{g^*} - \Delta_\rho^{\frac{s}{2}} u_{\tau}^{*}\lfloor_{\Omega_n}\rda_{\Lp{2}(\mu_n)} \leq C\eps \]
with probability at least $1-Cn^{-\alpha}-Cne^{-cn\eps^{d+4s}}$.
\end{proposition}

The proof of the proposition is given in Section~\ref{subsubsec:VarL2:DisCtsNoiseless:L2VarProof}.
The proof of Theorem~\ref{thm:Intro:Res:VarL2:VarBound} follows immediately from the triangle inequality and Propositions~\ref{prop:VarL2:Noise:NoiseDisEst} and~\ref{prop:VarL2:DisCtsNoiseless:NoiselessLim}.
One of the main ingredients for proving Proposition \ref{prop:VarL2:DisCtsNoiseless:NoiselessLim} is the following result which is of interest on its own. 

\begin{theorem}
\label{thm:VarL2:DisCtsNoiseless:PolyLapCons}
Let Assumptions (A1)-(A4) hold and $s\in \bbN$.
Define $\Delta_n$, $\Delta_\rho$ and $\sigma_\eta$ by~\eqref{eq:Intro:DisOp:Deltan}, \eqref{eq:Intro:ContOp:Delta} and~\eqref{eq:Intro:ContOp:sigmaeta} respectively.
Then, for any $\alpha>1$, there exists $C>c>0$ and $\eps_0>0$ such that for any $u\in\Ck{2s+1}$ and $\eps, n$ satisfying~\eqref{eq:Intro:Res:VarL2:epsBound}, 
\[ \lda \l \Delta_n^s - \Delta_\rho^s \r u \rda_{\Lp{2}(\mu_n)} \leq C \eps \l \|u\|_{\Ck{2s+1}(\Omega)} + 1 \r \]
with probability at least $1-Cn^{-\alpha} - Cne^{-cn\eps^{d+4s}}$.
\end{theorem}

We notice that when $s=1$ it is well known that the graph Laplacian is pointwise consistent and the rate at which it converges, e.g. \cite{Singer}.
Theorem \ref{thm:VarL2:DisCtsNoiseless:PolyLapCons} generalises this result, and states that with high probability $\Delta_n^s u \to \Delta_\rho^s u$ in an $\Lp{2}$ sense for all $s\in\bbN$ where $u$ is sufficiently smooth and $\eps = \eps_n$ satisfies a lower bound.
The proof of Theorem~\ref{thm:VarL2:DisCtsNoiseless:PolyLapCons} is given in Section~\ref{subsubsec:VarL2:DisCtsNoiseless:PolyLapConsProof}.

Before presenting a rigorous proof of Proposition~\ref{prop:VarL2:DisCtsNoiseless:NoiselessLim}, let us present a heuristic argument.
First, we write
\begin{align*}
\Delta_n^s u(x) - \Delta_\rho^s u(x) & = \Delta_n^{s-1} \l \Delta_n - \Delta_\rho\r v^{(0)}(x) + \l \Delta_n^{s-1}-\Delta_\rho^{s-1}\r v^{(1)}(x) \\
& = \Delta_n^{s-1} \l \Delta_n - \Delta_\rho\r v^{(0)}(x) + \Delta_n^{s-2} \l \Delta_n - \Delta_\rho\r v^{(1)}(x) + \l \Delta_n^{s-2}-\Delta_\rho^{s-2}\r v^{(2)}(x) \\
& = \dots \\
& = \sum_{k=1}^s \Delta_n^{s-k} \l \Delta_n - \Delta_\rho\r v^{(k-1)}(x)
\end{align*}
where $v^{(k)} = \Delta_\rho^k u$.
We keep track of higher order errors in the pointwise consistency of the graph Laplacian, following the method in~\cite{calder18} to estimate, when $v\in \Ck{r}$,
\begin{equation} \label{eq:VarL2:DisCtsNoiseless:FineErr}
\l \Delta_n - \Delta_\rho\r v(x) = \eps_n E_1(x) + \eps_n^2 E_2(x) + \dots \eps_n^{r-3} E_{r-3}(x) + \eps_n^{r-2} E_{r-2}(x)
\end{equation}
where $E_i\in \Ck{r-i-2}$.
Now, heuristically one expects (with high probability) $\|\Delta_n^{j} E_i\|_{\Lp{2}(\mu_n)}\lesssim \| E_i\|_{\Ck{2j}(\Omega)}$ (when $j\leq \frac{r-i-2}{2}$) and we recall a worse case (high probability) operator norm bound $\| \Delta_n^j \|_{\op} \leq C\eps_n^{-2j}$, see Lemma~\ref{lem:VarL2:Noise:OpBounds}.
Letting $u=u_\tau^*$, and assuming $g\in\Ck{1}(\Omega)$, we can immediately infer that $u\in\Ck{2s+1}(\Omega)$ from~\eqref{eq:VarL2:DisCtsNoiseless:ELcEinftytau} (as a standard elliptic regularity result).
We choose $v=v^{(k-1)}$ in~\eqref{eq:VarL2:DisCtsNoiseless:FineErr} and note that $r=2(s-k)+3$.
Now, (with high probability)
\begin{align*}
\lda \Delta_n^{s-k} E_i \rda_{\Lp{2}(\Omega)} & = \lda \Delta_n^{\frac{i-1}{2}} \Delta_n^{s-k-\frac{i-1}{2}} E_i \rda_{\Lp{2}(\Omega)} \\
 & \leq \lda\Delta_n\rda_{\op}^{\frac{i-1}{2}} \lda \Delta_n^{s-k-\frac{i-1}{2}} E_i\rda_{\Lp{2}(\Omega)} \\
 & \leq C\eps_n^{1-i} \lda E_i\rda_{\Ck{2(s-k)-i+1}(\Omega)}.
\end{align*}
So, (with high probability)
\begin{align*}
\lda \Delta_n^{s-k} \l \Delta_n - \Delta_\rho\r v^{(k-1)} \rda_{\Lp{2}(\mu_n)} & \leq \sum_{i=1}^{2(s-k)+1} \eps_n^i \lda \Delta_n^{s-k} E_i\rda_{\Lp{2}(\mu_n)} \\
 & \leq C\eps_n\sum_{i=1}^{2(s-k)+1} \| E_i\|_{\Ck{2(s-k)-i+1}(\Omega)} \\
 & \leq C \eps_n.
\end{align*}
Thus, $\lda \Delta_n^s u - \Delta_\rho^s u \rda_{\Lp{2}(\mu_n)} = O(\eps_n)$ (note that $C$ in the above inequality depends on $u$, in the proof we will show that this dependence is in terms of $\|u\|_{\Ck{2s+1}(\Omega)}$, i.e. $\lda \Delta_n^s u - \Delta_\rho^s u \rda_{\Lp{2}(\mu_n)}\leq C\eps_n \l \|u\|_{\Ck{2s+1}(\Omega)}+1\r$).

The above discussion is clearly formal and we spend the remainder of the section making the proof rigorous. We do this in two stages.
The first step gives operator bounds on $\Delta_n$ for smooth functions, i.e. quantifying $\|\Delta_n^{j} E_i\|_{\Lp{\infty}(\mu_n)}\lesssim \| E_i\|_{\Ck{2j}(\Omega)}$.
The second step derives~\eqref{eq:VarL2:DisCtsNoiseless:FineErr} from which we can prove Theorem~\ref{thm:VarL2:DisCtsNoiseless:PolyLapCons} when combined with the first step.

\subsubsection{Operator Bounds on Powers of the Graph Laplacian} \label{subsubsec:VarL2:DisCtsNoiseless:OpBounds}

The aim of this subsection is to prove the following proposition.

\begin{proposition}
\label{prop:VarL2:DisCtsNoiseless:OpBounds:GraLapOpBounds}
Let Assumptions~\ref{ass:Intro:DisOp:A1}-\ref{ass:Intro:DisOp:A4} hold and $m\in\bbN$.
Define $\Delta_n$ by~\eqref{eq:Intro:DisOp:Deltan}.
Then, for all $\alpha>1$, there exists $C>c>0$ and $\eps_0>0$ such that for any $\eps, n$ satisfying~\eqref{eq:Intro:Res:VarL2:epsBound} and for all $v\in\Ck{2m}(\Omega)$ we have
\[ \| \Delta_n^m v\|_{\Lp{2}(\mu_n)} \leq C (\| v\|_{\Ck{2m}(\Omega)}+1) \]
with probability at least $1-Cne^{-cn\eps^{d+4m-2}}-Cn^{-\alpha}$.
\end{proposition}

Let us define the \emph{non-local continuum Laplacian} by
\begin{equation} \label{eq:VarL2:DisCtsNoiseless:OpBounds:NLLap}
\Delta_\eps v(x) = \frac{2}{\eps^2} \int_{\Omega} \eta_{\eps}(|x-y|) \l v(x) - v(y) \r \rho(y) \, \dd y. 
\end{equation}
We prove the proposition in two steps.
In the first step we show $\| \Delta_\eps^m v\|_{\Lp{2}(\Omega)}\leq C \| v\|_{\Ck{2m}(\Omega)}$.
In the second step  we bound the difference $\|\Delta_n^m v - \Delta_\eps^m v\|_{\Lp{2}(\mu_n)}$.
Initially we consider the case when $m=1$, which is just the difference of $\Delta_n v(x)$ to its expected value $\Delta_\eps v(x) = \bbE\ls \Delta_n v(x)\rs$.
We can then bootstrap this to $m>1$.
Putting the two steps together proves Proposition~\ref{prop:VarL2:DisCtsNoiseless:OpBounds:GraLapOpBounds}.

\begin{lemma}
\label{lem:VarL2:DisCtsNoiseless:OpBounds:NLLapOpBounds}
Let Assumptions~\ref{ass:Intro:DisOp:A1}, \ref{ass:Intro:DisOp:A3} and~\ref{ass:Intro:DisOp:A4} hold, and $k\in\bbN$.
Define $\Delta_\eps$ by~\eqref{eq:VarL2:DisCtsNoiseless:OpBounds:NLLap}.
Then, there exists $C>0,\eps_0>0$ such that for all $\eps\in(0,\eps_0)$ and for all $v\in\Ck{k+2}(\Omega)$ we have
\begin{equation} \label{eq:VarL2:DisCtsNoiseless:OpBounds:NLLapOpBounds}
\lda \Delta_\eps v \rda_{\Ck{k}(\Omega)} \leq C\| v\|_{\Ck{k+2}(\Omega)}. 
\end{equation}
Furthermore, if
$v\in\Ck{2k}(\Omega)$ then
\begin{equation} \label{eq:VarL2:DisCtsNoiseless:OpBounds:NLLapOpBounds-2} 
\lda \Delta_\eps^k v \rda_{\Ck{0}(\Omega)} \leq C\| v\|_{\Ck{2k}(\Omega)}. 
\end{equation}
\end{lemma}

\begin{proof}
We can write, for $\eps$ sufficiently small, where $\nabla$ above is the gradient in $\bbR^d$, and $D^2$ the matrix of second derivatives of  a function on $\bbR^d$,
\begin{align*}
\Delta_\eps v (x) & = \frac{2}{\eps^2} \int_{B(x,\eps)} \eta_\eps(|x-y|) (v(x) - v(y)) \rho(y) \,\dd y \\
 & = -\frac{2}{\eps^2} \int_{\bbR^d} \eta(|z|) \l \eps \nabla  v(x) \cdot z + \eps^2 \int_0^1 \int_0^t D^2 v(x+\eps sz)[z,z]\,\dd s \,\dd t \r \\
 & \qquad \times \l  \rho(x) + \eps\int_0^1 \nabla \rho(x+\eps sz) \cdot z \,\dd s\r \,\dd z,
\end{align*}
by Taylor's theorem and a change of variables.
Using the reflective symmetry of $\eta$ we have $\int_{\bbR^d} \eta(|z|) z \, \dd z=0$ and hence,
\begin{align*}
\Delta_\eps v (x) & = -2\nabla v(x) \cdot \int_{\bbR^d} \eta(|z|) z \int_0^1 \nabla \rho(x+\eps s z) \cdot z \, \dd s \, \dd z \\
 & \qquad - 2\rho(x) \int_{\bbR^d} \eta(|z|) \int_0^1 \int_0^t D^2v(x+\eps s z)[z,z] \, \dd s \,\dd t \, \dd z \\
 & \qquad - 2 \eps \int_{\bbR^d} \eta(|z|) \l \int_0^1 \int_0^t D^2v(x+\eps s z)[z,z] \, \dd s \, \dd t \r \l \int_0^1 \nabla \rho(x+\eps s z) \cdot z \, \dd s \r \, \dd z.
\end{align*}
If $v\in\Ck{k+2}(\Omega)$ and $\rho\in\Ck{k+1}(\Omega)$ then $\Delta_\eps v\in \Ck{k}(\Omega)$ and moreover
\begin{align*}
\|\Delta_\eps v\|_{\Ck{k}(\Omega)} & \leq C \l \|v\|_{\Ck{k+1}(\Omega)} \|\rho\|_{\Ck{k+1}(\Omega)} + \|v\|_{\Ck{k+2}(\Omega)} \|\rho\|_{\Ck{k}(\Omega)} + \eps \|v\|_{\Ck{k+2}(\Omega)} \|\rho\|_{\Ck{k+1}(\Omega)} \r \\
 & \leq C \|v\|_{\Ck{k+2}(\Omega)}.
\end{align*}
This proves the first part of the lemma.
Iterating the estimate \eqref{eq:VarL2:DisCtsNoiseless:OpBounds:NLLapOpBounds} implies \eqref{eq:VarL2:DisCtsNoiseless:OpBounds:NLLapOpBounds-2}.
\end{proof}

Now we turn to Step 2 and bounding the difference $\Delta_n-\Delta_\eps$. 

\begin{lemma}
\label{lem:VarL2:DisCtsNoiseless:OpBounds:DeltanToDeltaeps}
Let Assumptions~\ref{ass:Intro:DisOp:A1}-\ref{ass:Intro:DisOp:A4} hold.
Define $\Delta_n$ by~\eqref{eq:Intro:DisOp:Deltan} and $\Delta_\eps$ by~\eqref{eq:VarL2:DisCtsNoiseless:OpBounds:NLLap}.
For any $\eps_0>0$ there exists $C>c>0$ such that for any $\eps\in (0,\eps_0)$, $p>0$, $n\in\bbN$ and $w\in\Ck{1}(\Omega)$ we have
\[ \sup_{x\in\Omega_n} \la \l \Delta_n - \Delta_\eps\r w(x) \ra \leq \eps^p \|w\|_{\Ck{1}(\Omega)} \]
with probability at least $1-Cne^{-cn\eps^{d+2p+2}}$.
\end{lemma}

\begin{proof}
Fix $w\in\Ck{1}(\Omega)$, $x\in\Omega_n$ and let $\Xi_i = \frac{2}{\eps^2} \eta_\eps(|x-y|) \l v(x)-v(y)\r$.
So,
\begin{equation} \label{eq:VarL2:DisCtsNoiseless:OpBounds:Xi}
\frac{1}{n}\sum_{i=1}^n \Xi_i = \Delta_n w(x) \qquad \text{and} \qquad \bbE[\Xi_i] = \Delta_\eps w(x).
\end{equation}
Note that
\[ \la \Xi_i-\bbE[\Xi_i]\ra \leq \frac{C\|w\|_{\Ck{1}(\Omega)}}{\eps^{d+1}} \qquad \text{and} \qquad \bbE\ls \Xi_i - \bbE[\Xi_i]\rs^2 \leq \frac{C\|w\|_{\Ck{1}(\Omega)}^2}{\eps^{d+2}}. \]
By Bernstein's inequality for any $t>0$,
\[ \bbP\l \sum_{i=1}^n \l \Xi_i - \bbE[\Xi_i] \r \geq t \r \leq \exp\l-\frac{ct^2\eps^{d+2}}{n\|w\|_{\Ck{1}(\Omega)}^2 + t\eps \|w\|_{\Ck{1}(\Omega)}}\r. \]
Choosing $t = n\eps^{p}\|w\|_{\Ck{1}(\Omega)}$ implies
\[ \bbP\l \sum_{i=1}^n \l \Xi_i - \bbE[\Xi_i] \r \geq n\eps^{p}\|w\|_{\Ck{1}(\Omega)} \r \leq \exp\l-\frac{cn\eps^{d+2p+2}}{1 + \eps^{p+1}}\r \leq \exp\l-cn\eps^{d+2p+2}\r. \]
Symmetrising the argument we have
\[ \la \sum_{i=1}^n \l \Xi_i - \bbE[\Xi_i] \r \ra \leq n\eps^{p}\|w\|_{\Ck{1}(\Omega)} \]
with probability at least $1-2e^{-cn\eps^{d+2p+2}}$.
Substituting in~\eqref{eq:VarL2:DisCtsNoiseless:OpBounds:Xi} and union bounding over all $x\in\Omega_n$ we have proved the lemma.
\end{proof}

Using the above lemma we can provide a bound on $\Delta_n^m-\Delta_\eps^m$.

\begin{lemma}
\label{lem:VarL2:DisctsNoiseless:OpBounds:LapBound}
Assume Assumptions~\ref{ass:Intro:DisOp:A1}-\ref{ass:Intro:DisOp:A4} hold and $m\in\bbN$.
Define $\Delta_n$ by~\eqref{eq:Intro:DisOp:Deltan} and $\Delta_\eps$ by~\eqref{eq:VarL2:DisCtsNoiseless:OpBounds:NLLap}.
Then, for all $\alpha>1$, there exists $C>c>0$ and $\eps_0>0$ such that for any $\eps, n$ satisfying~\eqref{eq:Intro:Res:VarL2:epsBound} and $v\in\Ck{2m-1}(\Omega)$ we have
\[ \lda \Delta_n^m v - \Delta_\eps^m v \rda_{\Lp{2}(\mu_n)} \leq C\|v\|_{\Ck{2m-1}(\Omega)} \]
with probability at least $1-Cne^{-cn\eps^{d+4m-2}}-Cn^{-\alpha}$.
\end{lemma}

\begin{proof}
We can write
\begin{align*}
\lda \Delta_n^m v - \Delta_\eps^m v \rda_{\Lp{2}(\mu_n)} & \leq \sum_{i=0}^{m-1} \lda \Delta_n^{m-i} \Delta_\eps^i v - \Delta_n^{m-i-1} \Delta_\eps^{i+1} v \rda_{\Lp{2}(\mu_n)} \\
 & \leq \sum_{i=0}^{m-1} \| \Delta_n \|_{\op}^{m-i-1} \lda \l \Delta_n - \Delta_\eps\r \Delta_\eps^i v \rda_{\Lp{2}(\mu_n)} \\
 & \leq C\sum_{i=0}^{m-1} \| \Delta_\eps^i v\|_{\Ck{1}(\Omega)} \\
 & \leq C\sum_{i=0}^{m-1} \| v\|_{\Ck{2i+2}(\Omega)} \\
 & \leq C\|v\|_{\Ck{2m}(\Omega)}
\end{align*}
by Lemmas~\ref{lem:VarL2:Noise:OpBounds},  \ref{lem:VarL2:DisCtsNoiseless:OpBounds:NLLapOpBounds} and~\ref{lem:VarL2:DisCtsNoiseless:OpBounds:DeltanToDeltaeps} with probability at least $1-Cne^{-cn\eps^{d+4m-2}}-Cn^{-\alpha}$.
\end{proof}

\subsubsection{Proof of Theorem~\ref{thm:VarL2:DisCtsNoiseless:PolyLapCons}} \label{subsubsec:VarL2:DisCtsNoiseless:PolyLapConsProof}

Now, we note that
\begin{align*}
\Delta_n^s u(x) - \Delta_\rho^s u(x) & = \Delta_n^{s-1} \l \Delta_n - \Delta_\rho\r v^{(0)}(x) + \l \Delta_n^{s-1}-\Delta_\rho^{s-1}\r v^{(1)}(x) \\
& = \Delta_n^{s-1} \l \Delta_n - \Delta_\rho\r v^{(0)}(x) + \Delta_n^{s-2} \l \Delta_n - \Delta_\rho\r v^{(1)}(x) + \l \Delta_n^{s-2}-\Delta_\rho^{s-2}\r v^{(2)}(x) \\
& = \dots \\
& = \sum_{k=1}^s \Delta_n^{s-k} \l \Delta_n - \Delta_\rho\r v^{(k-1)}(x)
\end{align*}
where $v^{(i)} = \Delta_\rho^i u$.

The idea is now to use pointwise convergence but to keep track of higher order terms than the estimates that appear in~\cite{Singer,calder18}.
For example,~\cite{calder18} shows that if $f\in\Ck{3}$ then
\begin{equation} \label{eq:VarL2:DisCtsNoiseless:PolyLapConsProof:LapPoint}
\la \Delta_n f(x) - \Delta f(x) \ra \leq C\|f\|_{\Ck{3}} \vartheta,
\end{equation}
where $\vartheta\geq \eps$, with probability at least $1-Cne^{-cn\eps^{d+2}\vartheta^2}$.
Directly applying the operator bounds we have
\begin{align*}
\lda \Delta_n^s u - \Delta_\rho^s u \rda_{\Lp{2}(\mu_n)} & \leq  \sum_{k=1}^s \|\Delta_n \|_{\op}^{s-k} \lda \l\Delta_n - \Delta_\rho\r v^{(k-1)} \rda_{\Lp{2}(\Omega)} \\
& \leq C\sum_{k=1}^s \eps^{-2(s-k)} \| v^{(k-1)}\|_{\Ck{3}(\Omega)} \vartheta_k \\
& \leq  C\|u\|_{\Ck{2s+1}} \sum_{k=1}^s \eps^{-2(s-k)} \vartheta_k.
\end{align*}
If we could choose $\vartheta_k = \eps^{1+2(s-k)}$ then the proof is immediate; however the pointwise convergence result~\eqref{eq:VarL2:DisCtsNoiseless:PolyLapConsProof:LapPoint} requires $\vartheta\geq\eps$ which rules out this choice.
However, we will show that this gives the right answer, in particular, that the convergence is within $\eps$ with probability at least $1-Cne^{-cn\eps^{d+4s}}$.
The rest of the section is devoted to removing the assumption that $\vartheta_k\geq \eps$.

\begin{proof}[Proof of Theorem~\ref{thm:VarL2:DisCtsNoiseless:PolyLapCons}]
Let us fix $k$ and write $v = v^{(k-1)}$.
Then, assuming $u\in \Ck{2s+1}(\Omega)$ we have $v\in \Ck{2(s-k)+3}(\Omega)$ and so, for $y$ sufficiently close to $x$,
\[ v(y) = v(x) + \sum_{j=1}^{2(s-k+1)} \sum_{i^{(j)}\in\{1,\dots,d\}^j} a_{i^{(j)}}^{(j)} \prod_{\ell = 1}^j \l y_{i_\ell^{(j)}} - x_{i_\ell^{(j)}} \r + O\l \la y_{i_\ell^{(j)}} - x_{i_\ell^{(j)}}\ra^{2(s-k)+3}\r \]
where
\[ a_{i^{(j)}}^{(j)} = \frac{1}{j!} \frac{\partial^j v}{\partial x_{i_1^{(j)}}\cdots \partial x_{i_j^{(j)}}}(x) \]
and $i^{(j)} = (i_1^{(j)},\dots,i_j^{(j)}) \in \{1,\dots d\}^j$.
Now we can write
\begin{align*}
\Delta_n v(x) & = \frac{2}{n\eps^2} \sum_{y\in\Omega_n} W_{xy} (v(x) - v(y)) \\
& = - \frac{2}{n\eps^2} \sum_{y\in\Omega_n} W_{xy} \ls \sum_{j=1}^{2(s-k+1)} \sum_{i^{(j)} \in \{1,\dots d\}^j} a_{i^{(j)}}^{(j)} \prod_{\ell=1}^j \l y_{i_\ell^{(j)}} - x_{i_\ell^{(j)}}\r \rs + O\l \frac{\eps^{2(s-k)+1}}{n} \sum_{y\in\Omega_n} W_{xy}\r.
\end{align*}
By Lemma~\ref{lem:VarL2:Noise:DegreeBounds}, $\frac{1}{n} \sum_y W_{xy} \leq C$ for all $x\in\Omega_n$ with probability at least $1-2ne^{-cn\eps^d}$, hence
we can write (with probability at least $1-2ne^{-cn\eps^d}$)
\[ \Delta_n v(x) = -\sum_{j=1}^{2(s-k+1)} \sum_{i^{(j)} \in \{1,\dots d\}^j} a_{i^{(j)}}^{(j)} I_{i^{(j)}}^{(j)} + O(\eps^{2(s-k)+1}) \]
where
\[ I_{i^{(j)}}^{(j)} = \sum_{y\in\Omega_n} \Psi_{i^{(j)}}^{(j)}, \qquad \Psi_{i^{(j)}}^{(j)}(y) = \frac{2}{n\eps^2} W_{xy} \prod_{\ell=1}^j \l y_{i_\ell^{(j)}}^{(j)} - x_{i_\ell^{(j)}}^{(j)} \r. \]
Note that $\|\Psi\|_{\Lp{\infty}}\leq \frac{C\eps^{j-2-d}}{n}$ and $\bbE[\Psi(Y)^2]\leq \frac{C\eps^{2(j-2)-d}}{n^2}$.
Hence, by Bernstein's inequality
\begin{align*}
I_{i^{(j)}}^{(j)} & = \frac{2}{\eps^2} \int_{\Omega} \eta_\eps(|x-y)|) \ls \prod_{\ell=1}^j \l y_{i_\ell^{(j)}} - x_{i_\ell^{(j)}} \r \rs \rho(y) \, \dd y + O(\eps^{j-2}\vartheta) \\
 & = 2\eps^{j-2} \int_{\bbR^d} \eta(|z|) \ls \prod_{\ell=1}^j z_{i_\ell^{(j)}}\rs \rho(x+\eps z) \, \dd z + O(\eps^{j-2}\vartheta)
\end{align*}
with probability at least $1-2ne^{-cn\eps^d\vartheta^2}$ for all $x\in\Omega_n$.
After union bounding we may assume that the above estimate holds for all $x\in\Omega_n$, for all $k=1,\dots,s$, for all $j=1,\dots, k$, and for all $i^{(j)}\in\{1,\dots,d\}^j$ with probability at least $1-Cne^{-cn\eps^d\vartheta^2}$.
We choose $\vartheta = \eps^{2(s-k)+3-j}$ and so, since $\vartheta\geq \eps^{2s}$, the following holds with probability at least $1-Cne^{-cn\eps^{d+4s}}$.

Now we approximate
\[ \rho(x+\eps z) = \sum_{m=0}^{2(s-k+1)-j} \eps^m \sum_{p^{(m)}\in \{1,\dots,d\}^m} b_{p^{(m)}}^{(m)} \prod_{q=1}^m z_{p_q^{(m)}} + O(\eps^{2(s-k)-j+3}) \]
where
\[ b_{p^{(m)}}^{(m)} = \frac{1}{m!} \frac{\partial^m \rho}{\partial x_{p_1^{(m)}} \cdots \partial x_{p_m^{(m)}}}(x). \]
Hence,
\[ I_{i^{(j)}}^{(j)} = 2 \sum_{m=0}^{2(s-k+1)-j} \sum_{p^{(m)}\in \{1,\dots,d\}^m}  \eps^{m+j-2} b_{p^{(m)}}^{(m)}  \int_{\bbR^d} \eta(|z|) \ls \prod_{\ell=1}^j z_{i_\ell^{(j)}}\rs \ls \prod_{q=1}^m z_{p_q^{(m)}} \rs \, \dd z + O(\eps^{(2(s-k)+1}). \]

Let
\[ F(j,m) = \sum_{i^{(j)}\in \{1,\dots,d\}^j} \sum_{p^{(m)}\in \{1,\dots,d\}^m} a_{i^{(j)}}^{(j)} b_{p^{(m)}}^{(m)} C(i^{(j)},p^{(m)}) \]
and
\[ C(i^{(j)},p^{(m)}) = -2 \int_{\bbR^d} \eta(|z|) \ls \prod_{\ell=1}^j z_{i_\ell^{(j)}} \rs \ls \prod_{q=1}^m z_{p_q^{(m)}} \rs \, \dd z \]
so that
\[ \Delta_n v(x) = \sum_{j=1}^{2(s-k+1)} \sum_{m=0}^{2(s-k+1)-j} \eps^{m+j-2} F(j,m) + O(\eps^{(2(s-k)+1}). \]

We now look at the following terms: (i) $j=1$, $m=0$; (ii) $j=1$, $m=1$; and $j=2$, $m=0$ (the terms which are potentially of order $\eps^{-1}$ and $\eps^0$).
For (i),
\[ C(i,\emptyset) = -2 \int_{\bbR^d} \eta(|z|) z_i \, \dd z = 0. \]
For (ii),
\[ C(i,p) = -2 \int_{\bbR^d} \eta(|z|) z_iz_p \, \dd x = \lb \begin{array}{ll} 0 & \text{if } i\neq p \\ -2\sigma_\eta & \text{if } i=p. \end{array} \rd \]
For (iii),
\[ C((i_1,i_2),\emptyset) = -2 \int_{\bbR^d} \eta(|z|) z_{i_1} z_{i_2} \, \dd z =  \lb \begin{array}{ll} 0 & \text{if } i_1\neq i_2 \\ -2 \sigma_\eta & \text{if } i_1=i_2. \end{array} \rd \]
So $F(1,0) = 0$,
\[ F(1,1) = -2 \sigma_\eta \sum_{i=1}^d a_i^{(1)} b_i^{(1)} = -2\sigma_\eta\nabla v(x) \cdot \nabla \rho(x), \]
and
\[ F(2,0) = -2 \sigma_\eta \sum_{i=1}^d a_{i,i}^{(2)} b^{(0)} = -\sigma_\eta \rho(x) \trace(D^2 v(x)). \]
As $F(1,0)\eps^{-1} + F(1,1) + F(2,0) = -\frac{\sigma_\eta}{\rho(x)} \Div(\rho^2 \nabla v)(x) = \Delta_\rho v(x)$ then we have (adding back the $k$ dependence on $v$)
\begin{align*}
\Delta_n v^{(k-1)}(x) - \Delta_\rho v^{(k-1)}(x) & = \sum_{m=2}^{2(s-k)+1} \eps^{m-1} F(1,m) + \sum_{m=1}^{2(s-k)} \eps^{m} F(2,m) \\
& \qquad + \sum_{j=3}^{2(s-k+1)} \sum_{m=0}^{2(s-k+1)-j} \eps^{m+j-2} F(j,m) + O(\eps^{2(s-k)+1}).
\end{align*}
In particular, if we let $F_{j,m}^{(k)}(x) = F(j,m)$ and $O(\eps^{2(s-k)+1}) = \eps^{2(s-k)+1} E^{(k)}(x)$ then
\begin{align*}
\lda \l \Delta_n^s - \Delta_\rho^s\r u \rda _{\Lp{2}(\mu_n)} & \leq \sum_{k=1}^s \lda \Delta_n^{s-k} \l \Delta_n-\Delta_\rho \r  v^{(k-1)}\rda_{\Lp{2}(\mu_n)} \\
 & \leq \sum_{k=1}^s \sum_{m=2}^{2(s-k)+1} \eps^{m-1} \| \Delta_n^{s-k} F_{1,m}^{(k)} \|_{\Lp{2}(\mu_n)} + \sum_{k=1}^s \sum_{m=1}^{2(s-k)} \eps^{m}  \| \Delta_n^{s-k} F_{2,m}^{(k)} \|_{\Lp{2}(\mu_n)} \\
 & \qquad + \sum_{k=1}^s \sum_{j=3}^{2(s-k+1)} \sum_{m=0}^{2(s-k+1)-j} \eps^{m+j-2}  \| \Delta_n^{s-k} F_{j,m}^{(k)} \|_{\Lp{2}(\mu_n)} \\
 & \qquad + \sum_{k=1}^s \eps^{2(s-k)+1} \| \Delta_n^{s-k} E^{(k)} \|_{\Lp{2}(\mu_n)}.
\end{align*}

By Lemma~\ref{lem:VarL2:Noise:OpBounds} (with probability at least $1-Cn^{-\alpha}$) we have
\[ \eps^{2(s-k)+1} \| \Delta_n^{s-k} E^{(k)} \|_{\Lp{2}(\mu_n)} \leq \eps^{2(s-k)+1} \| \Delta_n \|_{\op}^{s-k} \| E^{(k)} \|_{\Lp{2}(\mu_n)} \leq \eps \| E^{(k)} \|_{\Lp{2}(\mu_n)}. \]

We also have $F_{j,m}^{(k)}\in\Ck{2(s-k)+3-j}$ and $\|F_{j,m}^{(k)}\|_{\Ck{2(s-k)+3-j}(\Omega)} \leq C\|u\|_{\Ck{2s+1}(\Omega)}$, therefore we have for $j\geq 3$
\begin{align*}
\eps^{m+j-2} \| \Delta_n^{s-k} F_{j,m}^{(k)} \|_{\Lp{2}(\mu_n)} & \leq \eps^{m+j-2} \| \Delta_n\|_{\op}^{\frac{j-3}{2}} \| \Delta_n^{\frac{2(s-k)+3-j}{2}} F_{j,m}^{(k)} \|_{\Lp{2}(\mu_n)} \\
 & \leq C \eps^{m+1} \l \|F_{j,m}^{(k)}\|_{\Ck{2(s-k)+3-j}(\Omega)} + 1 \r
\end{align*}
with probability at least $1-Cn^{-\alpha}-Cne^{-cn\eps^{d+4m-2}}\geq 1-Cn^{-\alpha}-Cne^{-cn\eps^{d+2s}}$ by Lemma~\ref{lem:VarL2:Noise:OpBounds} and Proposition~\ref{prop:VarL2:DisCtsNoiseless:OpBounds:GraLapOpBounds}.
When $j=1,2$ we have, directly from Proposition~\ref{prop:VarL2:DisCtsNoiseless:OpBounds:GraLapOpBounds},
\[ \|  \Delta_n^{s-k} F_{j,m}^{(k)} \|_{\Lp{2}(\mu_n)} \leq \| F_{j,m}^{(k)} \|_{\Ck{2(s-k)}(\Omega)} \leq \| F_{j,m}^{(k)} \|_{\Ck{2(s-k)+3-j}(\Omega)} \leq C \| u\|_{\Ck{2s+1}(\Omega)} \]
with probability at least $1-Cne^{-cn\eps^{d+2s}}$.
Hence,
\[ \lda \l \Delta_n^s - \Delta_\rho^s\r u \rda _{\Lp{2}(\mu_n)} \leq C\eps \l \| u\|_{\Ck{2s+1}(\Omega)} + 1 \r \]
with probability at least $1-Cn^{-\alpha} - Cne^{-cn\eps^{d+4s}}$.
\end{proof}

\begin{remark}
In our proofs we avoid attempting to establish pointwise consistency results for the difference $\Delta_n^s - \Delta_\rho^s$ (for arbitrary $s\in\mathbb{N}$) when acting on smooth enough functions, and instead by careful manipulation of the equations, we rely only on the existing pointwise consistency results for the case $s=1$~\cite{HeinvonLuxburgAudibert,calder18}.
\end{remark}

\subsubsection{Proof of Proposition \ref{prop:VarL2:DisCtsNoiseless:NoiselessLim}}  \label{subsubsec:VarL2:DisCtsNoiseless:L2VarProof}

We start with two preliminary lemmas which will be used in the proof of Proposition~\ref{prop:VarL2:DisCtsNoiseless:NoiselessLim}.

\begin{lemma}
\label{lem:VarL2:DisCtsNoiseless:L2VarProof:SolBound}
Let $\tau>0$, $s>0$, $\Delta_n$ be defined by~\eqref{eq:Intro:DisOp:Deltan} and $\Delta_\rho$ defined by~\eqref{eq:Intro:ContOp:Delta} where $\sigma_\eta$ is defined by~\eqref{eq:Intro:ContOp:sigmaeta}.
Assume $w_n$ and $w$ solve
\begin{align*}
\tau \Delta_n^s w_n + w_n & = h_n \\
\tau \Delta_\rho^s w + w & = h
\end{align*} 
for $h_n\in\Lp{2}(\mu_n)$ and $h\in\Lp{2}(\mu)$.
Then,
\begin{align*}
\|w_n\|_{\Lp{2}(\mu_n)} & \leq \|h_n\|_{\Lp{2}(\mu_n)} \\
\|w\|_{\Lp{2}(\mu)} & \leq \|h\|_{\Lp{2}(\mu)}.
\end{align*}
\end{lemma}

\begin{proof}
Let $\{q_i^{(n)}\}_{i=1}^n$ be an eigenbasis of $\Delta_n$ with non-negative eigenvalues $\{\lambda_i^{(n)}\}_{i=1}^n$.
Then $w_n$ solving $\tau \Delta_n^s w_n + w_n = h_n$ implies
\[ \l\tau [\lambda_i^{(n)}]^s + 1 \r \langle w_n,q_i^{(n)}\rangle_{\Lp{2}(\mu_n)} = \langle h_n, q_i^{(n)}\rangle_{\Lp{2}(\mu_n)}. \]
So,
\[ \|w_n\|^2_{\Lp{2}(\mu_n)} = \sum_{i=1}^n \la \langle w_n,q_i^{(n)}\rangle_{\Lp{2}(\mu_n)} \ra^2 = \sum_{i=1}^n \la \frac{\langle h_n,q_i^{(n)}\rangle_{\Lp{2}(\mu_n)}}{1+\tau [\lambda_i^{(n)}]^s} \ra^2 \leq \sum_{i=1}^n \la \langle h_n,q_i^{(n)}\rangle_{\Lp{2}(\mu_n)} \ra^2 = \|h_n\|^2_{\Lp{2}(\mu_n)}. \]
The proof for $\|w\|_{\Lp{2}(\mu)} \leq \|h\|_{\Lp{2}(\mu)}$ is analogous.
\end{proof}

\begin{lemma}
\label{lem:VarL2:DisCtsNoiseless:L2VarProof:utau*Bound}
Assume Assumptions~\ref{ass:Intro:DisOp:A3} and~\ref{ass:Intro:DisOp:A5} hold and $s>0$.
Define $\Delta_\rho$ by~\eqref{eq:Intro:ContOp:Delta} where $\sigma_\eta$ is defined by~\eqref{eq:Intro:ContOp:sigmaeta}.
Let $u_\tau^*$ be the solution to~\eqref{eq:VarL2:DisCtsNoiseless:ELcEinftytau}.
Then, for all $\tau_0>0$ there exists $C$ such that
\[ \sup_{\tau\in (0,\tau_0)} \| u_\tau^{*}\|_{\Ck{2s+1}(\Omega)} \leq C. \]
\end{lemma}

\begin{proof}
Let $\{(\lambda_i,q_i)\}_{i=1}^\infty$ be eigenpairs of $\Delta_\rho$.
Define $\cH^k(\Omega) = \lb u\in\Lp{2}(\mu) \,:\, \sum_{i=1}^\infty \lambda_i^k \langle u,q_i\rangle_{\Lp{2}(\mu)}^2 < +\infty\rb$ with the norm $\|u\|_{\cH^k(\Omega)}^2 = \sum_{i=1}^\infty \lambda_i^k \langle u,q_i\rangle_{\Lp{2}(\mu)}^2$.
And let $\Hk{k}(\Omega)$ be the usual Sobolev space with square integrable $k$th (weak) derivative.
By~\cite[Lemma 17]{dunlop18} $\cH^k(\Omega)\subseteq \Hk{k}(\Omega)$ and there exists $C>c>0$ (depending only on the choice of $k$) such that
\[ C\| u\|_{\Hk{k}(\Omega)} \geq \|u\|_{\Hk{k}(\Omega)} \geq c\| u\|_{\cH^k(\Omega)} \qquad \text{for all } u \in\cH^k(\Omega). \]
As in the proof of Lemma~\ref{lem:VarL2:DisCtsNoiseless:L2VarProof:SolBound} we take advanatge of the fact that $\langle u_\tau^*,q_i\rangle_{\Lp{2}(\mu)} = \frac{\langle g,q_i\rangle_{\Lp{2}(\mu)}}{1+\tau\lambda_i^s}$ to infer
\begin{align*}
\| u_\tau^*\|_{\cH^k(\Omega)}^2 & = \sum_{i=1}^\infty \lambda_i^k \langle u_\tau^*,q_i\rangle_{\Lp{2}(\mu)}^2 \\
 & = \sum_{i=1}^\infty \lambda_i^k \la \frac{\langle g,q_i\rangle_{\Lp{2}(\mu)}}{1+\tau \lambda_i^s} \ra^2 \\
 & \leq \sum_{i=1}^\infty \lambda_i^k \langle q,q_i\rangle_{\Lp{2}(\mu)}^2 \\
 & = \| g\|_{\cH^k(\Omega)}^2 \\
 & \leq C^2 \|g\|_{\Hk{k}(\Omega)}^2.
\end{align*}
Hence $\|u_\tau^*\|_{\Hk{k}(\Omega)} \leq \frac{C}{c} \|g\|_{\Hk{k}(\Omega)}$.
By choosing $k$ sufficiently large and employing Morrey's inequality we can find $\bar{C}$ such that $\|u\|_{\Ck{2s+1}(\Omega)}\leq \bar{C} \|u\|_{\Hk{k}(\Omega)}$ for all $u\in \Hk{k}(\Omega)$.
In particular, $\|u_\tau^*\|_{\Ck{2s+1}(\Omega)} \leq \frac{C\bar{C}}{c} \|g\|_{\Hk{k}(\Omega)}$ which proves the lemma.
\end{proof}

We can now prove Proposition~\ref{prop:VarL2:DisCtsNoiseless:NoiselessLim}.

\begin{proof}[Proof of Proposition \ref{prop:VarL2:DisCtsNoiseless:NoiselessLim}]
We have
\[ \tau \Delta_n^s u_\tau^* + u_\tau^* - g = \tau \l \Delta_n^s - \Delta_\rho\r u_\tau^*. \]
So, letting $w= u_{\tau}^{*}\lfloor_{\Omega_n} - u_{n,\tau}^{g^*}$ we can bound
\[ \tau \Delta_n^s w + w = \tau \l \Delta_n^s - \Delta_\rho\r u_\tau^*. \]
By Lemma~\ref{lem:VarL2:DisCtsNoiseless:L2VarProof:SolBound} and~Theorem~\ref{thm:VarL2:DisCtsNoiseless:PolyLapCons}
\[ \|w\|_{\Lp{2}(\mu_n)} \leq \tau \lda \l \Delta_n^s - \Delta_\rho^s\r u_\tau^* \rda_{\Lp{2}(\mu_n)} \leq C\tau \eps \l \|u_\tau^{*}\|_{\Ck{2s+1}(\Omega)} + 1 \r \]
with probability at least $1-Cn^{-\alpha} - Cne^{-cn\eps^{d+4s}}$.
By Lemma~\ref{lem:VarL2:DisCtsNoiseless:L2VarProof:utau*Bound} $\|u_\tau^{*}\|_{\Ck{2s+1}(\Omega)}$ can be bounded for all $\tau\in(0,\tau_0)$.
This completes the proof of the first inequality.

We can derive the second inequality from the first inequality and Theorem~\ref{thm:VarL2:DisCtsNoiseless:PolyLapCons} as follows
\begin{align*}
\lda \Delta_n^{\frac{s}{2}} u_{n,\tau}^{g^*} - \Delta_\rho^{\frac{s}{2}} u_{\tau}^{*}\lfloor_{\Omega_n}\rda_{\Lp{2}(\mu_n)}^2 & \leq 2\lda \Delta_n^{\frac{s}{2}} u_{n,\tau}^{g^*} - \Delta_n^{\frac{s}{2}} u_{\tau}^{*}\rda_{\Lp{2}(\mu_n)}^2 + 2 \lda \Delta_n^{\frac{s}{2}} u_{\tau}^{*} - \Delta_\rho^{\frac{s}{2}} u_{\tau}^{*}\lfloor_{\Omega_n}\rda_{\Lp{2}(\mu_n)}^2 \\
 & \leq 2 \lda \Delta_n^{s} \l u_{n,\tau}^{g^*} - u_{\tau}^{*}\lfloor_{\Omega_n} \r \rda_{\Lp{2}(\mu_n)} \lda u_{n,\tau}^{g^*} - u_{\tau}^{*}\lfloor_{\Omega_n}\rda_{\Lp{2}(\mu_n)} \\
 & \qquad \qquad + C\eps^2\l \|u_\tau^{*}\|_{\Ck{s+1}(\Omega)} + 1\r^2 \\
\end{align*}
with probability at least $1-Cn^{-\alpha}-Cne^{-cn\eps^{d+4s}}$.
Comparing the Euler-Lagrange equations we have
\[ \tau \Delta_n^s \l u_{n,\tau}^{g*} - u_\tau^*\lfloor_{\Omega_n} \r + \l u_{n,\tau}^{g*} - u_\tau^* \r = \tau\l\Delta_\rho^s - \Delta_n^s \r u_\tau^*. \]
By Theorem~\ref{thm:VarL2:DisCtsNoiseless:PolyLapCons}  we can derive the bound
\begin{align*}
\lda \Delta_n^s \l u_{n,\tau}^{g*} - u_\tau^*\lfloor_{\Omega_n} \r \rda_{\Lp{2}(\mu_n)} & \leq \frac{1}{\tau} \lda u_{n,\tau}^{g*} - u_\tau^*\lfloor_{\Omega_n}\rda_{\Lp{2}(\mu_n)} + \lda \l\Delta_\rho^s - \Delta_n^s \r u_\tau^* \rda_{\Lp{2}(\mu_n)} \\
 & \leq C\eps\l 1+\|u_\tau^*\|_{\Ck{2s+1}(\Omega)} \r
\end{align*}
with probability at least $1-Cn^{-\alpha}-Cne^{-cn\eps^{d+4s}}$.
Therefore,
\[ \lda \Delta_n^{\frac{s}{2}} u_{n,\tau}^{g^*} - \Delta_\rho^{\frac{s}{2}} u_{\tau}^{*}\lfloor_{\Omega_n}\rda_{\Lp{2}(\mu_n)}^2 \leq C\eps^2 \l \tau + (1+\tau) \|u^*_\tau\|_{\Ck{2s+1}(\Omega)} \r \]
with probability at least $1-Cn^{-\alpha}-Cne^{-cn\eps^{d+4s}}$.
If $\tau\leq \tau_0$ then we can bound by $C\eps^2$ as required.
\end{proof}

Putting together Proposition~\ref{prop:VarL2:Noise:NoiseDisEst} and Proposition~\ref{prop:VarL2:DisCtsNoiseless:NoiselessLim} proves Theorem~\ref{thm:Intro:Res:VarL2:VarBound} and Remark~\ref{rem:Intro:Res:VarL2:VarDeltaBound}.

\section{\texorpdfstring{$\Lp{2}$}{L2} Bias Estimates} \label{sec:BiasL2}

Recalling that the Fr\'echet derivative of $\cEinftyg$ is
\[ \frac12 \nabla_{\Ltwo(\mu)} \cEinftyg(u) = \tau\Delta_\rho^s u + u - g \]
then one can easily check that the following subgradient equality holds
\begin{equation} \label{eq:BiasL2:SubGradEq}
\langle \nabla_{\Ltwo(\mu)} \cEinftyg(w), w-v\rangle_{\Ltwo(\mu)} - \| w-v\|_{\Ltwo(\mu)}^2 - \tau \|\Delta_\rho^{\frac{s}{2}} (w-v) \|_{\Ltwo(\mu)}^2 = \cEinftyg(w) - \cEinftyg(v)
\end{equation}
for any $v,w \in \Hs(\Omega)$.
Since $\nabla_{\Ltwo(\mu)} \cEinftyg(u_\tau^*) = 0$ and $g$ is sufficiently regular then
\begin{align*}
& \| u_\tau^* -g\|_{\Ltwo(\mu)}^2 + \tau \| \Delta_\rho^{\frac{s}{2}} (u_\tau^* - g)\|_{\Ltwo(\mu)}^2 = \cEinftyg(g) - \cEinftyg(u_\tau^*) \\
& \hspace{1cm} = \langle \nabla_{\Ltwo(\mu)} \cEinftyg(g), g-u_\tau^*\rangle_{\Ltwo(\mu)} - \| g-u_\tau^*\|_{\Ltwo(\mu)}^2 - \tau \|\Delta_\rho^{\frac{s}{2}} (g-u_\tau^*) \|_{\Ltwo(\mu)}^2
\end{align*}
where for the first equality we let $w=u_\tau^*$, $v=g$ in~\eqref{eq:BiasL2:SubGradEq}, and in the second equality we let $w=g$, $v=u_\tau^*$ in~\eqref{eq:BiasL2:SubGradEq}.
Hence,
\begin{align*}
\| u_\tau^* -g\|_{\Ltwo(\mu)}^2 + \tau \| \Delta_\rho^{\frac{s}{2}} (u_\tau^* - g)\|_{\Ltwo(\mu)}^2 & = \frac12\langle \nabla_{\Ltwo(\mu)} \cEinftyg(g), g-u_\tau^*\rangle_{\Ltwo(\mu)} \\
& \leq \frac12 \| \nabla_{\Ltwo(\mu)} \cEinftyg(g)\|_{\Ltwo(\mu)} \|g-u_\tau^*\|_{\Ltwo(\mu)}.
\end{align*}
It follows that
\[ \| u_\tau^* -g\|_{\Ltwo(\mu)}\leq \frac12 \| \nabla_{\Ltwo(\mu)} \cEinftyg(g)\|_{\Ltwo(\mu)} = \tau \|\Delta_\rho^s g\|_{\Ltwo(\mu)} \]
and
\[ \| \Delta_\rho^{\frac{s}{2}} (u_\tau^* - g)\|_{\Ltwo(\mu)}^2 \leq \frac12 \| \nabla_{\Ltwo(\mu)} \cEinftyg(g)\|_{\Ltwo(\mu)} \|\Delta_\rho^s g\|_{\Ltwo(\mu)} = \frac{\tau}{2} \|\Delta_\rho^s g\|_{\Ltwo(\mu)}^2 \]
which proves Theorem~\ref{thm:Intro:Res:BiasL2:BiasBound} and Remark~\ref{rem:Intro:Res:BiasL2:BiasDeltaBound}.

\section*{Acknowledgements}

NGT was supported by NSF Grant DMS 1912802.
MT was supported by the European Research Council under the European Union’s Horizon 2020 research and innovation programme Grant Agreement No. 777826 (NoMADS).

%\nocite{*}
\bibliographystyle{plain}
%\bibliography{references_all}
\bibliography{references}

\end{document}